\documentclass{article}

\usepackage{microtype}
\usepackage{graphicx}
\usepackage{booktabs} 

\usepackage{hyperref}

\usepackage[accepted]{icml2025_nodeco}

\usepackage{amsmath}
\usepackage{amssymb}
\usepackage{mathtools}
\usepackage{amsthm}

\usepackage[capitalize,noabbrev]{cleveref}

\theoremstyle{plain}
\newtheorem{theorem}{Theorem}[section]
\newtheorem{proposition}[theorem]{Proposition}
\newtheorem{lemma}[theorem]{Lemma}

\theoremstyle{definition}
\newtheorem{definition}[theorem]{Definition}
\newtheorem{assumption}[theorem]{Assumption}
\theoremstyle{remark}

\usepackage[textsize=tiny]{todonotes}

\usepackage{enumitem}
\usepackage{multirow}
\usepackage{tikz}
\usepackage{pgfplots}
\pgfplotsset{compat=1.18}
\usepackage{subcaption}
\usepackage{comment}
\usepackage{breakurl}

\newcommand{\R}{\mathbb{R}}    
\newcommand{\Rp}{\R_{>0}}      
\newcommand{\Rn}{\R_{\ge 0}}   
\renewcommand{\b}{\mathbf}     
\newcommand{\tb}{\textbf}      
\renewcommand{\d}{\mathrm{d}}  
\newcommand{\E}{\mathbb{E}}    
\renewcommand{\t}{\mathrm}     
\newcommand{\mD}{\mathcal{D}} 
\newcommand{\mP}{\mathcal{P}}  
\newcommand{\mS}{\mathcal{S}}  
\newcommand{\mA}{\mathcal{A}}  
\newcommand{\mL}{\mathcal{L}}  
\newcommand{\mT}{\mathcal{T}}  
\newcommand{\mM}{\mathcal{M}}  
\newcommand{\mG}{\mathcal{G}}  %
\newcommand{\mF}{\mathcal{F}}  
\newcommand{\mC}{\mathcal{C}} 

\DeclareMathOperator*{\argmin}{arg\,min}

\icmltitlerunning{To Switch or Not to Switch? Balanced Policy Switching in Offline Reinforcement Learning}

\begin{document}

\twocolumn[
\icmltitle{To Switch or Not to Switch?\\ Balanced Policy Switching in Offline Reinforcement Learning}



\icmlsetsymbol{equal}{*}

\begin{icmlauthorlist}
\icmlauthor{Tao Ma}{lse}
\icmlauthor{Xuzhi Yang}{lse}
\icmlauthor{Zolt{\'a}n Szab{\'o}}{lse}
\end{icmlauthorlist}

\icmlaffiliation{lse}{Department of Statistics,
  London School of Economics,
  Houghton Street, London, WC2A 2AE, UK}

\icmlcorrespondingauthor{Tao Ma}{t.ma9@lse.ac.uk}

\icmlkeywords{Machine Learning, ICML}

\vskip 0.3in
]



\printAffiliationsAndNotice{}  

\begin{abstract}
Reinforcement learning (RL)---finding the optimal behaviour (also referred to as policy) maximizing the collected long-term cumulative reward---is among the most influential approaches in machine learning with a large number of successful applications. In several decision problems, however, one faces the possibility of policy switching---changing from the current policy to a new one---which incurs a non-negligible cost, 
and in the decision one is limited to using historical data without the availability for further online interaction. Despite the inevitable importance of this offline learning scenario, to our best knowledge, very little effort has been made to tackle the key problem of balancing between the gain and the cost of switching in a flexible and principled way. Leveraging ideas from the area of optimal transport, we initialize the systematic study of policy switching in offline RL. We establish fundamental properties and design a Net Actor-Critic algorithm for the proposed novel switching formulation. Numerical experiments demonstrate the efficiency of our approach on multiple robot control benchmarks of the Gymnasium and traffic light control from SUMO-RL.
\end{abstract}

\section{Introduction} \label{sec:intro}

Reinforcement learning \citep[RL,][]{puterman2014markov} is a fundamental tool in machine learning for advising agents to make sequential decisions, which has recently witnessed an unprecedented breakthrough from both theoretical and application perspective \citep{sutton2018reinforcement}. Successful applications of RL include for instance beating human expert players in games \citep{mnih2013playing, silver2017mastering}, dynamic treatment and automated medical diagnosis in healthcare \citep{yu2021reinforcement, sun2024accountability}, robotics behaviour improvement \citep{kober2013reinforcement, tang2024deep} and autonomous driving \citep{kiran2021deep}. Due to its flexible design, RL is able to accommodate various important forms of optimal decision making.

In a broad sense, RL problems can be divided into two groups, online and offline RL, each of which has its distinct strengths and limitations. In the online setting, the agent can actively explore the unknown environment by executing actions according to the policies, and make use of the received rewards to adjust the behaviour for a higher future gain \citep{arulkumaran2017deep}. However, in scenarios where random exploration may be impractical or even dangerous \citep{singla2021reinforcement}, gathering a static dataset is often a more adequate choice. Motivated by such constraints, offline RL has emerged as a promising approach \citep{levine2020offline}. In the offline setting, some policies have already been applied in the environment and generated a large offline dataset. With such data, the agent cannot make further exploration, but is supposed to learn a better policy solely based on the available information \citep{haarnoja2018soft, fujimoto2019off, kumar2020conservative,
matsushima2021deploymentefficient,
fujimoto2021minimalist, kostrikov2021offline, an2021uncertainty, ma2024framework}. Due to the discrepancies between the policies that generated the offline data and the policy learned by an offline algorithm, solving decision problems offline is highly challenging, with expected sub-optimal performance \citep{kumar2019stabilizing} compared to their online counterparts.

Despite the success of RL algorithms in the offline setting \citep{kumar2020conservative, kostrikov2021offline}, one key but moderately studied question is the cost of policy switching. Significant cost can occur when changing from an old policy to a new one. One common example is the adaptive policies in traffic light control \cite{lopez2018microscopic}. Policy adjustment can lead to additional delays or even compromise traffic safety \cite{han2023leveraging}. So a policy switch for such traffic control introduces significant costs (such as temporary traffic congestion, update of facilities). Other examples include the cost of updating hardware devices \citep{mirhoseini2017device}, the fees to employ human annotators for large models \citep{he2023annollm}, the reorganization expenses of a company \citep{lopucki2004determinants}, or the additional efforts to modify webpage designs \citep{theocharous2015ad}. However, modelling such policy switching cost is a highly non-trivial task. For example, in the perspective of employees in a company, learning a new skill normally requires more efforts than relocating to a new team with similar tasks. Such scenario of strategy change and cost management is called organizational change management in the theory of business \citep{by2005organisational, lauer2010change}. On the other hand, in the existing literature of RL to our best knowledge, the focus was only on somewhat simplistic schemes of costs, which include the global and the local switching cost \citep{bai2019provably, gao2021provably, wang2021provably, qiao2022sample}. Both  definitions target to measure if two policies (or policies conditional on states) are the same or not, but ignore \emph{how} the two (families of) policies are \emph{different} from each other. In addition, these costs with limited forms of expressiveness were developed for the online setting.

In this work we focus on the offline RL setting. Our aim is  to initialize the formulation and understanding of the key properties of policy switching in this scenario. Throughout the paper, we will consider the following prototype offline RL task: the agent has been relying on an old policy for a long term, with which rich offline data has been generated. Now in the beginning of a new episode, there is one chance for the agent to execute a policy with the possibility of switching to a new one, where the change can have a non-negligible cost. Our \tb{goals} are three-fold:
\begin{enumerate}[labelindent=0em,leftmargin=1.3em,topsep=0cm,partopsep=0cm,parsep=0cm,itemsep=1mm]
    \item How to rigorously formulate such offline policy switching problem, and balance between the potential gain and the cost?
    \item Is there a way to construct a family of switching costs that are flexible and expressive?
    \item How to design an algorithm to robustly find a better policy in the new problem formulation?
\end{enumerate}

Given these three questions, our \tb{contributions} can be summarized as follows.

\begin{enumerate}[labelindent=0em,leftmargin=1.3em,topsep=0cm,partopsep=0cm,parsep=0cm,itemsep=1mm]
    \item We propose a new policy switching problem, by defining the novel net values and net Q-functions, and establish their fundamental properties which are in sharp contrast to their classic RL counterparts. 
    \item Motivated by mass transportation, we propose a flexible class of cost functions, which includes former definitions (local and global costs) as special cases.
    \item An algorithm, named Net Actor-Critic (NAC), is proposed to find a new policy which improves the old policy towards the optimal in terms of net value, which is the first offline method for policy switching problem with costs to our best knowledge. 
\end{enumerate}

The paper is structured as follows. We begin with preliminaries on notations, classic RL settings and a review of former switching costs in Section~\ref{sec:preliminaries}. In Section~\ref{sec: formulation} we introduce the notions of net value and net Q-function, with which the novel policy switching problem is formulated. A new family of cost functions are also provided relying on optimal transport (OT). We present our NAC algorithm to approximate the switch-optimal policy  in Section~\ref{sec: NAC}; the numerical efficiency of the approach is demonstrated in Section~\ref{sec:numerical-experiments}. 
Further algorithmic details, extensions of the problem formulation, proofs and implementation details of experiments are provided in the Appendix.

\section{Preliminaries} \label{sec:preliminaries}
In this section we provide the necessary background for the manuscript. Notations are introduced in Section~\ref{sec: notations}, and the classic RL settings and formerly proposed policy switching costs are elaborated in Section~\ref{sec: RL}.

\subsection{Notations}\label{sec: notations}
We introduce a few notations used throughout the paper. A $\sigma$-algebra on a set $X$ is denoted by $\Sigma_X$. Given measurable spaces $(X, \Sigma_X)$ and $(Y, \Sigma_Y)$, $(X \times Y, \Sigma_X \otimes \Sigma_Y)$ is the product space, where $\Sigma_X \otimes \Sigma_Y$ is the smallest $\sigma$-algebra generated by $\{A \times B: A \in \Sigma_X, B \in \Sigma_Y\}$. The set of all probability measure on $(X,\Sigma_X)$ is denoted by $\mP(X)$. Let $\mathcal{F}$ be the collection of all real-valued functions on $X$, $\|f\|_{\infty}:= \sup_{x \in X}|f(x)|$ ($f\in \mathcal{F}$), and define $\mG (X, \|\cdot\|_{\infty}) : = \{f\in \mF: \|f\|_{\infty}<+\infty\}$; $\mG (X, \|\cdot\|_{\infty})$ is known to be complete. For any set $A\subseteq X$, $I_A : X \to \{0, 1\}$ is the indicator function of $A$: $I_A(x) = 1$ if $x \in A$, $I_A(x) =0$ otherwise. For a set   $B$, $|B|$ stands for its cardinality. The set of non-negative real numbers is denoted by $\Rn$; similarly, $\Rp$ stands for the set of positive reals. Let $\mathrm{Id}$ be the identity map. For any positive integer $K$, $[K]:=\{1, \ldots, K\}$. For any $a, b \in \R$,  $a \wedge b := \min\{a, b\}$. A map $T$ from a metric space $(Z, \rho)$ into itself is called contraction if there exists a constant $c_T\in [0,1)$ such that $\rho\bigl(T(z_1), T(z_2)\bigr) \leq c_T \rho(z_1, z_2)$ for all $z_1, z_2 \in Z$.

\subsection{Classic RL Settings}\label{sec: RL}
In this subsection, we recall a few fundamental concepts of RL from the formulation of MDPs, alongside with the formerly proposed policy switching costs.

\begin{figure*}
    \centering
\begin{subfigure}[b]{.45\linewidth}
\centering
    \begin{tikzpicture}[scale = 1.3]
\draw[thick, ->] (-0.3, 0) -- (3.2, 0);
\draw[thick, ->] (-0.3, -0.8) -- (3.2, -0.8); 
\draw[thick, ->] (-0.3, -1.4) -- (3.2, -1.4); 
\draw[thick, ->, purple] (3.3, 0.1) -- (3.3, -1.5); 

\node[] at (1.5,-0.3) {$\ldots \quad \ldots \quad \ldots$};
\node[font = \tiny, anchor=center] at (1.4, 0.2) {$\hspace{-0.5cm} s_0^1 \hspace{0.5cm} a_0^1 \hspace{0.5cm} r_0^1 \hspace{0.4cm} \ldots \ \ldots \hspace{0.5cm} s_H^1$};
\node[font = \tiny, anchor=center] at (1.4, -0.6) {$\hspace{-0.3cm}\hspace{0cm} s_0^{K-1} \hspace{0.1cm} a_0^{K-1} \hspace{0.1cm} r_0^{K-1} \hspace{0cm} \ldots \ \ldots \hspace{0.4cm} s_H^{K-1}$};
\node[font = \tiny, anchor=center] at (1.4, -1.2) {$\hspace{-0.6cm} s_0^K \hspace{0.4cm} a_0^K \hspace{0.4cm} r_0^K \hspace{0.8cm} \ldots \ \ldots $};
\node[font = \small] at (-0.7, 0) {$\pi_1$};
\node[font = \small] at (-0.6, -0.8) {$\pi_{K-1}$};
\node[font = \small] at (-0.7, -1.4) {$\pi_{K}$};
\node[font = \small, rotate = 270, purple] at (3.5, -0.7) {Online learning};
\end{tikzpicture}
\caption{Episodic online learning.}
\end{subfigure}
\hspace{0.5cm}
\begin{subfigure}[b]{.45\linewidth}
\centering
\begin{tikzpicture}[scale = 1.3]
\draw[thick, ->] (-0.3, 0) -- (3.2, 0);
\draw[thick, ->] (-0.3, -0.8) -- (3.2, -0.8); 
\draw[thick, dashed, ->] (-0.3, -1.4) -- (3.2, -1.4); 

\node[] at (1.5,-0.3) {$\ldots \quad \ldots \quad \ldots$};
\node[font = \tiny] at (1.4, 0.2) {$\hspace{-0.5cm} s_0^1 \hspace{0.5cm} a_0^1 \hspace{0.5cm} r_0^1 \hspace{0.4cm} \ldots \ \ldots \hspace{0.5cm} s_H^1$};
\node[font = \tiny] at (1.4, -0.6) {$\hspace{-0.3cm}\hspace{0cm} s_0^{K-1} \hspace{0.1cm} a_0^{K-1} \hspace{0.1cm} r_0^{K-1} \hspace{0cm} \ldots \ \ldots \hspace{0.4cm} s_H^{K-1}$};
\node[font = \tiny] at (1.4, -1.2) {$\hspace{-0.6cm} s_0^K \hspace{0.4cm} a_0^K \hspace{0.4cm} r_0^K \hspace{0.8cm} \ldots \ \ldots  $};

\draw [decorate,decoration={brace,amplitude=5pt,raise=1pt},yshift=120pt]
node [font = \tiny, black,midway,yshift=-4.7cm, xshift =-0.7cm] {$\pi_{\mathrm{o}}$} (-0.3,-5.1) -- (-0.3,-4.2) ;
\node[font = \small, fill = cyan!30] at (-0.8, -1.4) {$\pi_{\mathrm{n}} = ?$};

\node[font = \small, rotate=270, purple] at (3.5, -0.2) {Offline data};
\draw[purple] (-0.9,-0.9) rectangle (3.3, 0.4);
\end{tikzpicture}
\caption{Policy switching problem based on offline data.}
\label{fig:policyswitchingOffline}
\end{subfigure}
\caption{Comparison between previous setting of online learning and ours.}
\label{fig:online-vs-offline}
\end{figure*}
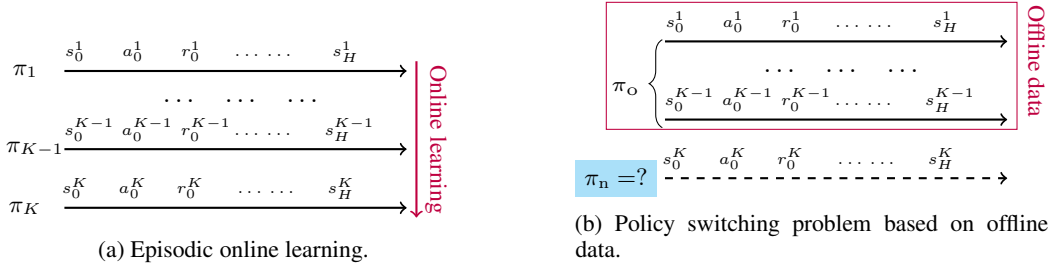

\tb{MDPs.} We consider a time-homogeneous 
MDP, denoted by ${\mM} = (\mS, \mA, P, R, \gamma)$, with $(\mS, \Sigma_\mS)$ and $(\mA, \Sigma_\mA)$ being measurable state and action spaces, respectively. 
Given any pair $(s,a)\in \mS \times \mA$, $P(\cdot | s, a): \Sigma_\mS \rightarrow [0, 1]$ is the transition kernel and $r(\cdot|s,a)$ encodes a stochastic reward with mean $R(s, a)$ and bounded support. Finally, 
$\gamma \in [0, 1)$ is the discount factor for future rewards. Given an MDP, a policy $\pi=\{\pi(\cdot|s), s\in\mS\}$ of an agent is a collection of conditional distributions on $(\mA, \Sigma_\mA)$, and $\Pi$ is the collection of all policies. With these notations at hand, an MDP proceeds as follows. The agent starts at a fixed initial state $s_0\in \mS$. At any step $t\geq0$
, the agent is at state $s_t \in \mS$, selects action $a_t\sim\pi(\cdot|s_t)$, receives a reward $r_t\sim r(\cdot|s_t, a_t)$, and is transitioned to $s_{t+1}\sim P(\cdot|s_t, a_t)$, the process of which creates one transition tuple $(s_t, a_t, r_t, s_{t+1})\in \mS \times \mA \times \R \times \mS$.

\tb{Evaluation \& optimality.}
For the purpose of policy evaluation and optimization, the \emph{value} at state $s$ and the \emph{Q-function} at state-action pair $(s,a)$ of $\pi$ are respectively defined as 
\begin{align*}
    V^{\pi}(s) &:= \mathbb{E}_{\pi} \left\lbrace \sum_{t=0}^{\infty} \gamma^t r_t \bigg| s_0 =  s \right\rbrace,\\
    Q^{\pi}({s, a}) &:= \mathbb{E}_{\pi} \left\lbrace \sum_{t=0}^{\infty}  \gamma^t r_t \bigg| s_0 = s, a_0 = a \right\rbrace,
\end{align*}
where $\E_{\pi}[\cdot]$ denotes the expectation according to $\pi$. 
At a state $s$, the optimal value is $V^*(s):=\max_\pi V^\pi(s)$. With a state-action pair $(s,a)$, the optimal Q-function is $Q^*(s,a):=\max_\pi Q^\pi(s,a)$, which are both taken over all policies. With two policies $\pi_1,\pi_2\in \Pi$, we say that $\pi_1$ is at least as good as $\pi_2$ if $V^{\pi_1}(s)\geq V^{\pi_2}(s)$ for all $s\in\mS$. The optimal policy is then defined as one that is at least as good as any other policy.
It is known that there always exists an \emph{optimal policy}, and for any optimal policy $\pi^*$, $Q^{\pi^\ast}(s,a)=Q^*(s,a)$ for all $(s,a)\in\mS\times\mA$.

\textbf{Online \& offline RL.} We now provide a  description of online and offline RL for convenient comparison; see Fig.~\ref{fig:online-vs-offline} for a visual illustration. In a given MDP, online learning is when the agent is allowed to switch policies throughout the process of learning. In practice, one specific example is when there are a total number of $K$ episodes, and the agent can choose a new policy $\pi_k$ for the $k$-th episode ($k\in[K]$). The data directly generated by policy $\pi_k$ proposed by the agent can be collected in the episode $k$, which further helps learning the next policy  $\pi_{k+1}$. On the other hand, offline learning is when a fixed dataset, containing  transition tuples by following some policy $\pi_\t{o}$ not proposed by the agent, is provided to the agent. And one needs to learn a better policy $\pi_\t{n}$, to apply in the following steps, only using this dataset, without any further interaction with the environment.

\textbf{Switching cost.} The limited coverage of switching cost formulations in the literature \citep{bai2019provably, gao2021provably, wang2021provably, qiao2022sample}, to our best knowledge, all focus on the online learning setting. For a finite $\mS$, the formerly proposed global and local switching costs \citep{bai2019provably} with $K$ episodes are respectively
\begin{align}
    C^{\mathrm{gl}}(\pi_1, \ldots, \pi_K) &= \sum_{k = 1}^{K-1} I_{\{\pi_k \not= \pi_{k+1}\}},  \label{eq:global-cost}
    \\ 
    C^{\mathrm{loc}}(\pi_1, \ldots, \pi_K) &= \sum_{k = 1}^{K-1}\sum_{s\in \mS} I_{\{\pi_{k}(\cdot|s) \neq \pi_{k+1}(\cdot|s)\}}. \label{eq:local-and-global-cost}
\end{align}
As long as the policy is changed, global switching cost will increase by $1$, while the increase in local switching cost is determined by how many states on which the conditional distributions are changed, which can be seen as a more fine-grained version of the global cost. 

The \tb{primary challenges} tackled in this paper are two-fold. First, our goal is to address the offline setting where an agent is only allowed to switch policy once, and this switch has a non-negligible cost. Second, the local switching cost is agnostic w.r.t.\ how different two distributions $\pi(\cdot|s)$ and $\pi^\prime(\cdot|s)$ are (it increases by $1$ as long as they are not identical); we aim to take into account that distributions far away are expected to incur higher costs than two similar ones.

\section{Problem Formulation}\label{sec: formulation}
In order to address the challenges outlined in Section~\ref{sec:preliminaries}, we introduce  the net value and net Q-function, with which a novel policy switching problem in offline RL is proposed in Section~\ref{sec: switchproblem}. The considered switching cost family is detailed in Section~\ref{sec: cost}, which significantly extends the existing switching costs.


\subsection{The Policy Switching Problem} \label{sec: switchproblem}
This section is dedicated to the formulation of our policy switching problem based on two new notions (net value and net Q-function) introduced below, followed by establishing some of their fundamental theoretical properties.

\tb{The question.} Enriched with the general setting of RL (Section~\ref{sec: RL}), we consider the following scenario; see Fig.~\ref{fig:policyswitchingOffline} for an illustration. There is a known old policy $\pi_\t{o}$, which has already been applied for $K-1$ episodes and led to the forming of an offline dataset $\mD$ with size $n=(K-1)H$. Now at $s_0\in\mS$, to begin the last episode, the agent needs to choose a policy. In addition, the agent is given a switching cost function $C$;  $C(\pi_\t{o}, \pi_\t{n})$ measures the policy switching cost (from $\pi_\t{o}$ to $\pi_\t{n}$) incurred in the beginning of the last episode.  This gives rise to the switching extension captured by the augmented tuplet ${\mM} = (\mS, \mA, P, R, \gamma,s_o, \pi_\t{o}, C)$.

There are two fundamental questions to be addressed:
\begin{enumerate}[labelindent=0em,leftmargin=1.3em,topsep=0cm,partopsep=0cm,parsep=0cm,itemsep=1mm]
    \item Is it profitable to switch to a different policy $\pi_\t{n}$ from the old $\pi_\t{o}$?
    \item In the case of switching, which new policy $\pi_\t{n}$ would better balance between the discounted total return in the last episode and the cost?
\end{enumerate}

We use the following two new notions to address these questions.
\begin{definition}[Net Value, Net Q-function]
For $s\in \mS, a\in \mA$, define the net value function and the net Q-function as
\begin{align*}
    V_N^{\pi_\t{n}}(s) &:=  V^{\pi_\t{n}}({s})  - C(\pi_\t{o}, \pi_\t{n}), \\
    Q_N^{\pi_\t{n}}(s,a) &:= Q^{\pi_{\t{n}}}(s, a) - C(\pi_\t{o},\pi_\t{n}).  
\end{align*}
\end{definition}
The value $V_N^{\pi_\t{n}}(s)$ measures after deducting the switching cost $C(\pi_\t{o},\pi_\t{n})$, the actual return in the last episode by adopting some new policy $\pi_\mathrm{n}$ and starting from state $s$.
Notice that the net value function is defined for all possible initial states $s\in \mS$ which will allow us to investigate optimality w.r.t. different initial states (Proposition~\ref{lemma: po_state}).
Using the analogue of business strategies, the one-time switching cost $C(\pi_\t{o}, \pi_\t{n})$ represents how much investment is needed to change to a new strategy $\pi_\t{n}$, the value $V^{\pi_\t{n}}({s})$ of a strategy is the total return in the future, while the net value $V_N^{\pi_\t{n}}(s)$ corresponds to the net income. The meaning of $Q_N^{\pi_\t{n}}(s,a)$ can be interpreted similarly.
In the traffic light control problem, the cost is due to the temporary congestion of the intersection, and the value is the evaluation of the efficiency for vehicles to pass the intersection.


Having defined net values, we now formulate the notion of switch-optimal policy while fixing the initial state $s_0 \in \mS$.
\begin{definition}[Switch-optimal policy]
Given an old policy $\pi_\t{o}$ and a fixed initial state $s_0 \in \mS$, a proposed policy $\pi_\t{n}^*$ is said to be switch-optimal if, for any policy $\pi_\t{c}\in\Pi$,
\begin{align}
    V_N^{\pi_\t{n}^*}(s_0) \geq V_N^{\pi_\t{c}}(s_0).
\end{align}
\end{definition}

Based on these definitions, our goal is to find a  switch-optimal policy $\pi_{\t{n}}^*$ or at least a policy $\pi_\t{n}$ which improves upon the old policy $\pi_\t{o}$ in terms of the net value function ($V_N^{\pi_\t{n}}(s_0) \geq V_N^{\pi_\t{o}}(s_0)$ where the r.h.s.\ equals to $V^{\pi_\t{o}}(s_0)$). If able to find such better $\pi_\t{n}$, the agent switches to this new policy; otherwise sticks with $\pi_\t{o}$ in the last episode. It should be noted that, although we try to find some policy close to the switch-optimal one, in the offline setting this can be rather challenging; so a new policy with significant improvement often already suffices. It is important to note that $\pi_{\mathrm{n}}^*$ need not have a significantly large value close to the optimal value $V^*(s_0)$, as our goal is not to find a policy with the maximal value function. Instead, we aim to find a policy that best balances the future return and the cost.     

Before moving on to our solution in the next section, we provide the following proposition for a deeper understanding of this new policy switching problem.  

\begin{assumption}\label{asmp: compact}
    The sets of values $\{V^\pi(s_0)\}_{\pi\in\Pi}$ and costs $\{C(\pi_\mathrm{o}, \pi)\}_{\pi\in\Pi}$ are compact.
\end{assumption}
Beyond existence, the following result shows various distinct characteristics [see Proposition \ref{lemma: nontrivial}-\ref{lemma: po_action}] specific to the switching setting.
\begin{proposition} \label{prop:meta}
For any MDP, the followings hold.
\begin{enumerate}[label=(\alph*), start = 1, ref=\thetheorem(\alph*),labelindent=0em,leftmargin=1.6em,topsep=0cm,partopsep=0cm,parsep=0cm,itemsep=1mm]
    \item\label{lemma: existence} If Assumption \ref{asmp: compact} is satisfied, then there always exists a switch-optimal policy.
    \item\label{lemma: nontrivial} There exists a cost function $C$, with which an optimal policy in value  is not switch-optimal in net value.
\end{enumerate}
If $\pi_{\mathrm{n}}^*$ is switch-optimal  in a fixed initial state ${s_0}$, then 
\begin{enumerate}[label=(\alph*), start = 3, ref=\thetheorem(\alph*),labelindent=0em,leftmargin=1.6em,topsep=0cm,partopsep=0cm,parsep=0cm,itemsep=1mm]
    \item if an alternative $s_0'\in \mS$ is fixed as initial state, then the switch-optimal policy may change.
    \label{lemma: po_state}
    \item it may not be the case that $Q_N^{\pi_{\t{n}}^*}(s_0, a)\geq Q_N^{\pi}(s_0, a)$ for all $a\in\mA$ and all  $\pi\in\Pi$. \label{lemma: po_action}
\end{enumerate}
\end{proposition}

\tb{Remarks:}
\begin{itemize}[labelindent=0em,leftmargin=1em,topsep=0cm,partopsep=0cm,parsep=0cm,itemsep=1mm]
    \item \tb{Existence}: Under mild assumptions, Proposition \ref{lemma: existence} guarantees the existence of a switch-optimal policy, which ensures that  the problem is well-posed. Proposition \ref{lemma: nontrivial} distinguishes the policy switching problem from the classic policy learning problem, as the respective optimal policies are different with an appropriate choice of $C$. It should be noted that the optimal polices in the two problems are not always different.
    \item \tb{Initial state dependence}: Proposition \ref{lemma: po_state} and \ref{lemma: po_action} indicate that the switch-optimal policy depends both on the initial state and the first action. This behaviour is in sharp contrast to the classic RL setting (Section~\ref{sec: RL}) where  an optimal policy achieves the highest value and Q-function \emph{simultaneously} on all states/state-action pairs. Such different characteristic of the optimal policies in the switching problem calls for a new approach to improve the candidate policy in the policy learning step of any proposed algorithm, as summing up returns over episodes with different initial states will be invalid in this case. 
\end{itemize}
    
As a useful computation tool for policy evaluation (used later in Algorithm~\ref{alg: ac}), we define the net Bellman operator and establish its contractive property.

\begin{definition}[Net Bellman operator]
Given any net Q-function $Q_N \in \mG(\mS \times \mA, \|\cdot\|_{\infty})$ and policy $\pi$, define the net Bellman operator $B^\pi : \mG(\mS \times \mA, \|\cdot\|_{\infty}) \to \mG(\mS \times \mA, \|\cdot\|_{\infty})$ of net Q-function as 
\begin{align*}
    (B^\pi Q_N)(s,a) 
    :=& R(s,a) - (1-\gamma)C(\pi_\text{o},\pi)\\ 
    &+ \gamma \mathbb{E}_{s^\prime \sim P(\cdot|s,a)} [V_N(s^\prime)], \\
    \text{with} \quad V_N(s) =& \mathbb{E}_{a\sim \pi(\cdot|s)} [Q_N(s,a)].
\end{align*}
\end{definition}

\begin{proposition}[Policy evaluation with net Q-function]\label{prop: evaluation}
    Given a net Bellman operator $B^\pi$ with respect to a policy $\pi$, and any net Q-function $Q_N^0 \in \mG(\mS \times \mA, \|\cdot\|_{\infty})$, let 
    \begin{align*}
    Q_N^{k+1}&:=B^\pi( Q_N^k)\,\, \text{  for } k = 0, 1, 2, \ldots
    \end{align*}
    Then $B^\pi$ is a contraction with parameter $c_{B^\pi}=\gamma$ and
    \begin{align*}
        \lim_{k\to\infty}Q_N^{k} = Q_N^\pi,
    \end{align*}
    where $Q_N^\pi \in \mG(\mS \times \mA, \|\cdot\|_{\infty})$ is the net Q-function of $\pi$.
\end{proposition}
Thanks to Proposition~\ref{prop: evaluation}, one can use the net Bellman operator to evaluate a given policy $\pi$ starting from an arbitrary net-Q function $Q_N^0$. In this work, we represent net Q-functions by neural networks, replace all expectations with sampled data and tune the parameters so that the net Bellman backup error $\|(B^\pi Q_N)-Q_N\|_2$ is small enough.

\subsection{The Family of Cost Functions} \label{sec: cost}
In this section, we first introduce two different components in the cost when switching from an old policy to a new one. Then we propose a general cost function family, which includes the reviewed local and global switching costs as specific cases. 
Finally we gradually zoom in to one specific choice of switching cost relying on optimal transport, which we also investigate numerically (Section~\ref{sec:numerical-experiments}).

\tb{Two components of switching cost.} In various policy switching problems, the induced switching costs come from two different sources: \emph{learning cost} and \emph{transaction cost}. Learning cost is incurred when the new policy introduces unfamiliar jobs, which requires serious effort to absorb. Meanwhile, transaction corresponds to the adjustment cost on existing familiar jobs. Such separation of costs have been a longstanding subject of analysis in economics \cite{nilssen1992two}. 
Taking traffic control as an example, one wants to update the control policies of some intersection to see if vehicles can pass the intersection more efficiently. Such switch can involve the temporary congestion due to configuring new equipments, which is the learning cost. While there is also influence to the traffic due to updating the software of the existing facilities or maintenance of the existing devices, which is the transaction cost.
These analogues are
reflected in the following cost family. 


\tb{General cost family.} We define a cost family 
\begin{align}
    C(\pi_{\t{o}}, \pi_{\t{n}}) \hspace{-0.1cm} := \sigma\Bigl(\int_{\mS} f({s}) &F(\pi_{\t{o}}(\cdot|s), \pi_{\t{n}}(\cdot|s)) \, \t{d}\mu({s})\Bigr), \label{def: AggregateCost}\\
    F(\pi_{\t{o}}(\cdot|s), \pi_{\t{n}}(\cdot|s)) :=& c_l\mL(\pi_{\t{o}}(\cdot|s), \pi_{\t{n}}(\cdot|s))\nonumber\\
    &+c_t\mT(\pi_{\t{o}}(\cdot|s), \pi_{\t{n}}(\cdot|s)), \label{def: AggregateCost:F}
\end{align}
with $\mL, \mT: \mP(\mA) \times \mP(\mA) \rightarrow \R$ capturing the learning cost and the transaction cost, with weights $c_l, c_t\in \R$, $f: \mS \to \R$ measurable function representing the relative importance weighting of different states, $\mu$ a probability measure on $\mS$, and activation function $\sigma: \R \to \R$; see Fig.~\ref{fig: Costfamily} for an illustration with finite state spaces ($|\mS|<\infty$). 

\begin{figure}
\centering
\begin{tikzpicture}[scale = 0.9]
\begin{axis}[
    no markers, 
    domain=-3:3, 
    samples=100,
    axis lines*=left, 
    xlabel={},
    ylabel={},
    every axis y label/.style={at=(current axis.above origin),anchor=south},
    every axis x label/.style={at=(current axis.right of origin),anchor=west},
    height=2.5cm, 
    width=8cm,
    xtick=\empty, 
    ytick=\empty,
    enlargelimits=false, 
    clip=false, 
    axis on top,
    grid = major,
    axis line style={draw=none}, 
    after end axis/.code={
        \draw (rel axis cs:0,0) -- (rel axis cs:1,0);
    }
    ]
    
    \addplot [very thick, red] {exp(-((x+0.7)^2))/sqrt(pi)}; 
    \addplot [very thick, blue] {exp(-((x-0.7)^2))/sqrt(pi)};
    
    \addplot [draw=none, fill=magenta!20] {exp(-((x+0.7)^2))/sqrt(pi)} \closedcycle;
    
    \addplot [draw=none, fill=green!30] {exp(-((x-0.7)^2))/sqrt(pi)} \closedcycle;

    \addplot [draw=none, fill=cyan!30, domain=-1.5:0] {min(exp(-((x+0.7)^2))/sqrt(pi), exp(-((x-0.7)^2))/sqrt(pi))} \closedcycle;

    \addplot [draw=none, fill=orange!70, domain=0:1.5] {min(exp(-((x+0.7)^2))/sqrt(pi), exp(-((x-0.7)^2))/sqrt(pi))} \closedcycle;
    
    \node at (axis cs:2,0.5) {$\pi_{\mathrm{n}}(\cdot | s)$};
    \node at (axis cs:-2,0.5) {$\pi_{\mathrm{o}}(\cdot |s)$};

    \draw[teal, ->, line width=1pt] (-0.7,0.4) node[anchor=north, xshift=5mm, yshift=8mm, font=\tiny] {Learning cost} to[out=60, in=120] (0.7,0.4);

    \draw[teal, <->, line width=1pt] (-0.9,0) node[anchor=north, xshift=0mm, yshift=-1mm, font=\tiny] {Transaction cost} to[out=-60, in=-120] (-0.2,0);

    \draw[teal, <->, line width=1pt] (0.2,0) node[anchor=north, xshift=7mm, yshift=-1mm, font=\tiny] {Transaction cost} to[out=-60, in=-120] (0.9,0);

    \addplot [thick, blue, domain=-1.5:0] {min(exp(-((x+0.7)^2))/sqrt(pi), exp(-((x-0.7)^2))/sqrt(pi))};

    \addplot [thick, red, domain=0:1.5] {min(exp(-((x+0.7)^2))/sqrt(pi), exp(-((x-0.7)^2))/sqrt(pi))};

    \draw [decorate,decoration={brace,amplitude=10pt,mirror,raise=4pt},yshift=-5pt]
    (axis cs:-3,0) -- (axis cs:0,0) node [black,midway,yshift=-0.7cm] {$\mA_1$};

    \draw [decorate,decoration={brace,amplitude=10pt,mirror,raise=4pt},yshift=-5pt]
    (axis cs:0,0) -- (axis cs:3,0) node [black,midway,yshift=-0.7cm] {$\mA_2$};
\end{axis}
\end{tikzpicture}
\caption{Transport switching cost.}\label{fig: TransportSWC}
\end{figure}

\begin{table*}
\caption{Choices of functions and parameters in the switching cost family.}
\label{tab: CoveredSC}
\begin{center}
\begin{tabular}{@{}l@{\hspace{0.2cm}}l@{\hspace{0.1cm}}cp{3.5cm}@{\hspace{0.1cm}}c@{\hspace{0.2cm}}c@{\hspace{0.2cm}}c@{\hspace{0.2cm}}c@{}}
\toprule
Cost & $\sigma(x)$ & $\mL(\pi_{\t{o}}(\cdot|s), \pi_{\t{n}}(\cdot|s))$& \multicolumn{1}{c}{$\mT(\pi_{\t{o}}(\cdot|s), \pi_{\t{n}}(\cdot|s))$}   & $f$ & $c_l$& $c_t$  & $\mu$  \\ 
\midrule 
   Local & $ |\mS|x$ & $ I_{\{\pi_{\t{o}}(\cdot|s) \not= \pi_{\t{n}}(\cdot|s)\}}$& \multicolumn{1}{c}{$0$}  & $1$ & $1$& $\R$ &  $\mathrm{Unif}(\mS)$\\
   Global & $ I_{\Rp}(x)$ &  $I_{\{\pi_{\t{o}}(\cdot|s) \not= \pi_{\t{n}}(\cdot|s)\}}$ & \multicolumn{1}{c}{$0$}  & $1$ & $1$& $\R$ & $\mathrm{Unif}(\mS)$\\
   \multirow{2}{*}{Transport}& \multirow{2}{*}{$\sigma(x)$} & \multirow{2}{*}{$|\pi_{\t{o}}(\mA_1|s) - \pi_{\t{n}}(\mA_1|s)|$} &  $\hspace{5pt} \pi_{\t{o}}(\mA_1|s) \wedge \pi_{\t{n}}(\mA_1|s)$ \newline $+  \pi_{\t{o}}(\mA_2|s) \wedge \pi_{\t{n}}(\mA_2|s)$  & \multirow{2}{*}{$f$} & \multirow{2}{*}{$\R$} & \multirow{2}{*}{$\R$}& \multirow{2}{*}{$\mu$}
   \\
\bottomrule
\end{tabular}
\end{center}
\end{table*}

The family \eqref{def: AggregateCost} subsumes various switching costs including the local and global ones with finite state space $|\mS|$ [as in \eqref{eq:global-cost} \& \eqref{eq:local-and-global-cost}]\footnote{For finite state space, one can choose $\mu$ to be the uniform distribution on $\mS$ and get back \eqref{eq:global-cost} and \eqref{eq:local-and-global-cost}.}; see  
Table~\ref{tab: CoveredSC}. This specialization also reveals that local/global costs can only measure the learning cost with simple indicator functions, thus are unable to capture the two different sources (learning and transaction) of the cost. 

\tb{Proposed transport switching cost.} We design  a decomposition of the state-wise cost $F$ specified in \eqref{def: AggregateCost:F} into the sum of learning cost $\mL$ and transaction cost $\mT$, relying on optimal transport (hence the name). We restrict our attention to a specific case of a more general construction (see the end of this section, and  Section~\ref{apx: GeneralTSC} for further details) (i) to keep the presentation simple, (ii) as it already conveys the key ideas, (iii) this specialization is easy-to-implement and already turns out to be beneficial as demonstrated by our numerical experiments on multiple RL benchmarks  (Section~\ref{sec:numerical-experiments}). 

In various decision problems the action space has a natural partition $\mA= \cup_{i = 1}^L \mA_i$, like the different skill sets in a department of a company. For easier understanding, we focus on the case of $L = 2$; see Fig.~\ref{fig: TransportSWC} for an illustration with colors indicating the different cost terms defined below. 

The construction consists of $2$ steps:

\tb{Step 1: Mass moving.} We move mass \emph{across} $\mA_1$ and $\mA_2$ such that the mass in each component $\mA_i$ agree. The amount needed to be moved is defined as the learning cost
\begin{align}
   \mL\bigl(\pi_{\t{o}}(\cdot|s), \pi_{\t{n}}(\cdot|s)\bigr) = {\color{magenta}|\pi_{\t{o}}(\mA_1|s) - \pi_{\t{n}}(\mA_1|s)|}. \label{def: LearningCost} 
\end{align}

\tb{Step 2: Mass rearrangement.} As some mass of $\pi_{\t{o}}(\cdot | s)$ remains in the same respective component during the first step---see the blue and orange areas in Fig.~\ref{fig: TransportSWC}---this part of mass will incur a cost due to rearrangement \emph{within} their own components, which gives rise to the transaction cost:
\begin{align}
    \mT\bigl(\pi_{\t{o}}(\cdot|s), \pi_{\t{n}}(\cdot|s)\bigr) =&~  {\color{cyan}  \pi_{\t{o}}(\mA_1|s) \wedge \pi_{\t{n}}(\mA_1|s)} \nonumber\\
    &+ {\color{orange}  \pi_{\t{o}}(\mA_2|s) \wedge \pi_{\t{n}}(\mA_2|s)}.\label{def: TransactionCost} 
\end{align}
\tb{Connection to OT.} The construction of~\eqref{def: LearningCost} and~\eqref{def: TransactionCost} implicitly defines a near-optimal transport map between $\pi_{\t{o}}(\cdot|s)$ and $\pi_{\t{n}}(\cdot|s)$, and serves as a tight upper bound for the optimal transport in classic OT theories \citep[see e.g.][Lemma 5.1]{staudt2023convergence}.
In addition, the definitions naturally extend to $L > 2$ by treating the mass transportation across components as learning cost, and rearrangement within components as transaction cost (definitions and details in Section \ref{apx: GeneralTSC}). And for any finite $L$ we show, in a wide range of settings, that the proposed cost is optimal.


\begin{proposition}[Optimality of the proposed cost]\label{prop:ot}
Given old policy $\pi_\t{o}$, new policy $\pi_\t{n}$ and any $\{\mA_i\}_{i = 1}^L$ as a partition of $\mA$. 
Let the transport cost (of the OT problem) be $c(x, y) = c_l \sum_{i \neq j} I_{\mA_i \times \mA_j}(x, y) + c_t \sum_{i = 1}^L I_{\mA_i \times \mA_i}(x, y)$ for $x,y\in\mA$, with $c_l \geq c_t$.
Then 
$F(\pi_\t{o}(\cdot|s), \pi_\t{n}(\cdot|s))$ solves the OT problem associated to $c$ for any $s \in \mS$.

\end{proposition}

Due to limited space, formal definitions of the OT problem and its optimality are available in Section~\ref{apx: GeneralTSC}. Note that the only assumption $c_l\geq c_t$ aligns with common scenarios in real life, when the learning cost due to unfamiliar jobs are expected to be higher than the transaction cost due to rearrangement. Such inequality also guides the choices of the $(c_l, c_t)$ pair in our numerical experiments. With Proposition \ref{prop:ot}, our proposed definition not only captures the two sources of costs, but also is (near) optimal, which could be hard to simultaneously achieve by other discrepancies like KL-divergence, total variation distance, or maximum mean discrepancy.


We briefly mention a further generalization of the transport switching cost. As justified by Proposition \ref{prop:ot}, in \eqref{def: LearningCost} and \eqref{def: TransactionCost}, we have implicitly used $I_{\mA_1 \times \mA_2}(a, a')$ as the  similarity measurement of actions. 
We note that one can also employ different measurements like the $L_2$ distance (see Section~\ref{apx: GeneralTSC}).


\section{Net Actor-Critic}
\label{sec: NAC}

In this section, we propose the Net Actor-Critic algorithm (NAC; Algorithm \ref{alg: ac}) to approximate the switch-optimal policy. Note that with known cost function that depends only on policies, actor-critic approach would separate the calculation of induced costs by actor from the conservative Q-function estimation, preventing inaccurate cost computation due to pessimism. At high level, NAC starts from evaluating the old policy, then alternately improves and evaluates the new policy in each iteration, and finally compares the empirical net values of the resulting new policy with the old one for a switching decision.

\begin{algorithm} 
   \caption{Net Actor-Critic (NAC)}\label{alg: ac}
{\bfseries Input:} Offline data $\mathcal{D}$, net Q-function parameters $\{\phi_i\}_{i\in[M]}$, target net Q-function parameters $\{\phi_j^\prime\}_{j\in[M]}$,  policy parameter $\theta$, and learning rates $\rho_{nq}, \rho_\theta, \rho_\text{stb}$.
\begin{algorithmic}[1]
    \STATE Apply Algorithm \ref{alg: eva} to evaluate old policy $\pi_\t{o}$
   \REPEAT
   \STATE Sample a mini-batch $B=\{(s, a, r, s^\prime)\}$ from $\mathcal{D}$
   \STATE Generate $a^\prime\sim\pi_\theta(\cdot|s^\prime)$, compute 
   $y(r, s^\prime)$ by \eqref{comp_tar}
   \STATE Update $Q_{N,\phi_i}$:
   $\phi_i \leftarrow  \phi_i - \rho_{nq} \nabla_{\phi_{i}} J_{Q_N,i}, i\in [M]$
   \STATE Improve policy with gradient ascent:
   $$\theta \leftarrow  \theta + \rho_\theta \nabla_{\theta} \mathbb{E}_{a\sim\pi_\theta(\cdot|s_0)} \left[\min\limits_{i\in[M]} Q_{N,\phi_i}(s_0, a)\right]$$
   \STATE Update $\phi_j^\prime$: $\phi_j^\prime \leftarrow  \rho_\text{stb} \phi_j^\prime + (1-\rho_\text{stb}) \phi_j, j\in[M]$
   \UNTIL{Stopping criterion met}
   \STATE (Optionally) apply Algorithm~\ref{alg: eva} to evaluate the resulting policy $\pi_\theta$
\end{algorithmic}
{\bfseries Output:}  $\pi_\text{out} = w(\pi_\text{o}, \pi_\theta)$
\end{algorithm}

\tb{Step 1: Old policy evaluation.} As a preliminary step, we need to evaluate the value of $\pi_\t{o}$, as a reference for later new policy training.  Since such algorithm is inspired by an offline fitted-Q evaluation \citep{munos2008finite,le2019batch}, sharing similar structure as the evaluation part in Algorithm~\ref{alg: ac}, due to limited space, we defer the presentation of Algorithm \ref{alg: eva} to Section~\ref{sec: details_alg}. 

\tb{Step 2: Off-policy evaluation.} With offline data $\mathcal{D}=\{(s_i, a_i, r_i, s_{i+1})\}_{i=1}^n$,  we first evaluate the net Q-function of the current policy $\pi_{\theta}$. Inspired by the pessimistic evaluation with clipped double Q-learning \citep{hasselt2010double, fujimoto2018addressing}, as well as the practical extension to multiple Q-evaluation \citep{an2021uncertainty}, we train $M$  net Q-functions, in the form of neural networks, in parallel, and take the minimum values to have a conservative estimation of the net Q-function. In addition, we also maintain separate target net Q-networks to improve the stability of evaluation process \citep{lillicrap2015continuous}. In each training iteration, we independently sample a mini-batch $B\subset\mathcal{D}$ instead of using the whole data. Hence, denoting the parameters of net Q-function by $\{\phi_i\}_{i\in[M]}$, that of the target net Q-functions by $\{\phi_i^\prime\}_{i\in[M]}$, and that of the policy by $\theta$, the target function for evaluation, calculated on $\{(s, a, r, s^\prime)\}$ is
\begin{align}
    y(r, s^\prime) := r 
    + \gamma  \min\limits_{i\in[M]}Q_{N,\phi_i^\prime}(s^\prime,a^\prime) 
    - (1-\gamma)&C(\pi_0,\pi_\theta), \nonumber\\
    \textrm{with}~a^\prime\sim \pi_\theta(\cdot|s^\prime)&.\label{comp_tar}
\end{align}
Then for each $i\in[M]$, we  update the parameter values $\phi_i$ using the gradient of
\begin{align*}
    J_{Q_N, i}:= \mathbb{E}_{(s, a, r, s^\prime) \sim B}[Q_{N,\phi_i}(s,a) - y(r,s^\prime)]^2.
\end{align*}
\tb{Step 3: Policy improvement.} We improve the policy by applying stochastic policy gradient ascent with the objective
\begin{align*}
\max\limits_{\theta}\E_{a\sim \pi_\theta(\cdot|s_0)} \left[\min\limits_{i\in[M]}Q_{N,\phi_i}(s_0,a)\right],
\end{align*}
which takes care of the state-dependent optimality, compared to former actor-critic methods.
By alternatively running Step 2 and 3 the policy is expected to improve towards the underlying switch-optimal one. Especially with multiple net Q-functions and the existence of costs, the distribution shift issue commonly observed in offline RL is naturally handled. To save computational efforts and avoid over-fitting, a stopping criterion is applied; for further details the reader is referred to Section~\ref{sec: details_alg}. 
Note that when the search process for the switch-optimal policy finishes, we can optionally further evaluate the found policy $\pi_\theta$ by Algorithm~\ref{alg: eva} to have more accurate offline evaluation.

\tb{Step 4: Final decision.} In the last step, the algorithm decides to switch to $\pi_\theta$ if its net value at $s_0$ exceeds the value of the old policy, and such final decision criterion can be defined through a decision function $w$, where
\begin{align*}
    w(\pi_\t{o}, \pi_\theta) \hspace{-0.05cm}:=\hspace{-0.05cm} I_{\{V^{\pi_\t{o}}(s_0) \geq V_N^{\pi_\theta}(s_0)\}}\pi_\text{o}
    \hspace{-0.1cm}
    +\hspace{-0.07cm} I_{\{V^{\pi_\t{o}}(s_0) < V_N^{\pi_\theta}(s_0)\}}\pi_\theta.
\end{align*}

\section{Numerical Experiments}\label{sec:numerical-experiments}

\begin{table*}[t]
\caption{Performance of NAC on Gymnasium benchmarks. 1st column: environment considered. 2nd column: dim($\mS$). 3rd column: dim($\mA$). 4th column: (sub)optimality of the old policy. 5-7th columns: performance measures, for ``Improvement" as mean $\pm$ std. The average net values of old policies are $-14.2$ (Ant-v4), $-60.5$ (HalfCheetah-v4), $15.6$ (Hopper-v4). }
\label{tab: results}
\begin{center}
\begin{tabular}{r@{\hspace{0.2cm}}r@{\hspace{0.1cm}}r@{\hspace{0.2cm}}r@{\hspace{0.2cm}}c@{\hspace{0.2cm}}c@{\hspace{0.2cm}}c}
\toprule
Environment $\mM$ & $d_\mS$ & $d_\mA$ & Old policy $\pi_\t{o}$ & Improvement& Switch proportion  & Responsible rate \\
\midrule 
        Ant-v4 & $27$ & $8$ &suboptimal & 58.2 $\pm$ 23.7 & 100.0\%  & 90.0\% \\
      HalfCheetah-v4 & 17 & 6 & suboptimal & 18.5 $\pm$ 19.5 & 80.0\%  & 70.0\% \\
            Hopper-v4 & 11 & 3 & suboptimal & 27.7 $\pm$ 16.9& 100.0\%   & 100.0\% \\
\midrule              
        Ant-v4                  & $27$ & $8$ & optimal & / & 0.0\%  & 100.0\% \\
       HalfCheetah-v4               & 17 &  6  & optimal & / & 0.0\%  & 100.0\%  \\
             Hopper-v4             & 11 & 3 & optimal & / & 20.0\%   & 80.0\% \\
\bottomrule
\end{tabular}
\end{center}
\end{table*}

\begin{table*}[t]
\caption{Performance of NAC on SUMO-RL benchmarks. 1st column: $c_l$. 2nd column: $c_t$. 3rd column: (sub)optimality of the old policy. 4-6th columns: performance measures, for ``Improvement" as mean $\pm$ std. The average net values of suboptimal old policies are $21.9$, $19.5$, $21.9$. }
\label{tab: results_sumo}
\begin{center}
\begin{tabular}{r@{\hspace{0.2cm}}l@{\hspace{0.2cm}}r@{\hspace{0.2cm}}c@{\hspace{0.2cm}}c@{\hspace{0.2cm}}c}
\toprule
$c_l$ & $c_t$ & Old policy $\pi_\t{o}$ & Improvement& Switch proportion  & Responsible rate \\
\midrule 
$0.5$ & $0.01$ &suboptimal &  $7.0$ $\pm$ $0.7$ & 100.0\%  & 100.0\% \\
$5.0$ & $0.1$ & suboptimal &  $6.2$ $\pm$ $0.4$ & 100.0\%  & 100.0\% \\
$10.0$ & $1.0$ & suboptimal &  $6.4$ $\pm$ $0.7$ & 100.0\%   & 100.0\% \\
\midrule              
$0.5$ & $0.01$ & optimal & / & 0.0\%  & 100.0\% \\
$5.0$ &  $0.1$  & optimal & / & 0.0\%  & 100.0\%  \\
$10.0$ & $1.0$ & optimal & / & 0.0\%   & 100.0\% \\
\bottomrule
\end{tabular}
\end{center}
\end{table*}

In this section we demonstrate the efficiency of the proposed NAC algorithm on various Gymnasium benchmarks \citep{towers_gymnasium_2023} for robot control and SUMO-RL \citep{sumorl} for traffic control.\footnote{All the code replicating our experiments is available at \url{https://github.com/xiaobaobaochifan/NAC}.} The experiments were designed to answer the following \tb{two questions} (in line with Section~\ref{sec: switchproblem}):
\begin{itemize}[labelindent=0em,leftmargin=2em,topsep=0cm,partopsep=0cm,parsep=0cm,itemsep=1mm]
    \item[Q1:] If the old policy $\pi_\t{o}$ is highly suboptimal in terms of its net value, can NAC find a new policy $\pi_\t{n}$ to improve it (in terms of net value)?
    \item[Q2:] When the old policy $\pi_\t{o}$ is already switch-optimal, will NAC advise the agent not to switch?
\end{itemize}
\textbf{Gymnasium.} We selected three environments of Gymnasium (version 0.29.1) to test these hypotheses and the performance of NAC: Ant-v4, HalfCheetah-v4 and Hopper-v4. Common characteristics of the environments are that 
\begin{itemize}[labelindent=0em,leftmargin=0.9em,topsep=0cm,partopsep=0cm,parsep=0cm,itemsep=1mm]
    \item their state and action spaces are continuous ($\mS \subseteq \R^{d_\mS}$, $\mA \subseteq \R^{d_\mA}$),
    \item the environments are challenging (due to their large dimensional state/action spaces; see Table~\ref{tab: results}),
    \item the aim of different 3D robots as agents in the environments is to fast move forward and remain healthy.
\end{itemize}
To simulate an already switch-optimal old policy $\pi_\t{o}$ (to Q2), we relied on the online version of the NAC algorithm. To obtain a highly sub-optimal old policy $\pi_\t{o}$ (to Q1), we initialized $\pi_\t{o}$ randomly for the HalfCheetah-v4 and the Hopper-v4 environment. For Ant-v4, most random policies were so weak that the agent could hardly learn anything useful from it, not to say improve. So we instead used a policy $\pi_\t{o}$ that was trained online for a few steps; this ensured that the agent could receive some positive rewards but $\pi_\t{o}$ was still far from optimal. For each environment and question (Q1 and Q2), we performed $10$ Monte Carlo experiments to assess the performance of NAC. In our experiments, we set $c_l=5$ and $c_t=0$ in the cost\footnote{The $c_t=0$ choice was made as it is the simplest setting which already goes beyond the local/global switching costs. Due to limited space, further results on $c_t\in\{0.1, 1\}$ are in Section \ref{sec: furtherexp}.}, and all the hyperparameters of the algorithms and parameters of the cost are provided in Section~\ref{sec: details_alg}. Note that we discard implementing the former methods with global/local costs, because 1) they are only for online setting; 2) those costs can only handle finite state space but we deal with more complex continuous spaces; 3) even in finite space case, any switch will lead to constant costs, which provides no information for policy learning.

Our performance measures reported (Table~\ref{tab: results}, with additional ablation study in Section~\ref{sec: furtherexp}) were as follows. With optimal old policies (Q2), we counted the proportion of repetitions over all random seeds when the algorithm advised the agent to switch; the perfect value is $0\%$. For suboptimal old policies (Q1), we calculated the same proportion (but the perfect value is $100\%$ instead). Such ratio is reported in the column with label ``Switch proportion". For all suboptimal cases, we report the mean $\pm$ std of the improvement in net value, with label ``Improvement". In addition, we also considered the performance measure ``Responsible rate". Recall that the NAC approach makes its decision by comparing the offline-estimated values of the old and newly-obtained policies. We also evaluated the two compared policies ($\pi_\t{o}$ and $\pi_\t{n}$) in an online fashion, providing a more accurate ``ground truth". 
The performance measure ``Responsible rate" counts the proportion that the decision made by NAC agrees with the one provided by the online evaluator. 
 
Table~\ref{tab: results} shows that for suboptimal old policies (Q1), in all environments NAC could significantly improve the net values (by relying on the offline data generated by such weak policies); the highest increase was $58.2$ in Ant-v4, noting that the average net value of old policy was only $-14.2$. In terms of switch decisions, in at least $80\%$ of the cases NAC advised the agent to switch to a new policy $\pi_\t{n}$; these results show that NAC encouraged the agent to explore better policies with high probability. For already optimal old policies (Q2), only in Hopper-v4 there were as low as $20\%$ of the cases where the algorithm advised to switch, while in other environments the decision was always to stick with the old one. Such high probabilities to keep old policies made sure that the agent did not switch to a less profitable policy. We can see that NAC provides responsible decisions in most cases: if due to randomness, the learned policy is not good enough, NAC will likely advise not to switch. 

\textbf{Traffic data.} To further showcase the applicability of NAC, we implemented the algorithm on the data from SUMO-RL \citep{sumorl}, an environment designed for developing and assessing traffic control algorithms in realistic urban scenarios.  The specific scenario concerned is an intersection with stochastically arriving vehicles. By observing the current light conditions, densities of vehicles and the numbers of queued ones, the agent can adaptively change the phases of lights in each direction to increase the speed for vehicles to pass and ease the pressure of the intersection. The instant reward is the efficiency of vehicles to cross the intersection. Switching to a new policy negatively influences such efficiency by delays, which is the cost.

According to Table \ref{tab: results_sumo}, the efficiency of vehicles near the intersection has been significantly increased, showcasing the effectiveness of our proposed method. Particularly, starting with sub-optimal old policies with net-value $\sim 20$, the improvement on average is $6.2-7.0$ (with std $0.4-0.7$). In addition, the switch proportion is $100\%$ when starting from sub-optimal old policies and $0\%$ when starting from an optimal one. In all cases, the responsible rate is $100\%$.

Further discussions of experiment settings, implementations and limitations are available in Section~\ref{sec: details_alg} and \ref{sec: furtherexp}.

\section*{Impact Statement}

We do not see any direct negative societal impact arising from the
proposed problem formulation of policy switching or the Net Actor-Critic algorithm.





\bibliography{main}
\bibliographystyle{icml2025}

\newpage
\appendix
\onecolumn

\setcounter{equation}{0}
\setcounter{figure}{0}
\renewcommand{\theequation}{\thesection\arabic{equation}}
\renewcommand{\thefigure}{\thesection\arabic{figure}}

\begin{center}
\Large \tb{Appendix}
\end{center}
\vspace{0.3cm}

In Section~\ref{apx: GeneralTSC}, we elaborate the general optimal transport based switching cost, which we specialized in the main body of the paper. Algorithmic details are provided in Section~\ref{sec: details_alg}. Additional experimental results are given in Section~\ref{sec: furtherexp}. Section~\ref{sec: proofs} is dedicated to proofs.

\section{General Formula for Transport Switching Cost}\label{apx: GeneralTSC}



In this section, we aim to provide generalizations for the transport switching cost in two directions: a) employ general measurements of similarity between two actions instead of indicator functions; b) consider the partition with multiple components, i.e. $L> 2$. The construction is inspired by a technique that are widely used to obtain the convergence rate of empirical Wasserstein distance \citep{fournier2015rate, weed2019sharp, lei2020convergence, staudt2023convergence}. 

\subsection{The Classic OT Theory}
Before delving into the details of the cost construction, we introduce two concepts from the optimal transport theory.   

\begin{definition}[Feasible transport plan]\label{def:feasible}
    Given any measure spaces $(X, \Sigma_X, \mu)$ and $(Y, \Sigma_Y, \nu)$. Then for any measure $\sigma$ on $(X \times Y, \Sigma_X \otimes \Sigma_Y)$, we say $\sigma$ is a feasible transport plan between $\mu$ and $\nu$ if for any $A \in \Sigma_X$ and any $B \in \Sigma_Y$ we have 
    \begin{align*}
        \sigma(A \times Y) = \mu(A), \quad \text{and} \quad \sigma(X \times B) = \nu(B),
    \end{align*}
and we write as $\sigma \in \mC(\mu, \nu)$.
\end{definition}

\begin{definition}[Optimal transport plan]
Given measure spaces $(X, \Sigma_X, \mu)$ and $(Y, \Sigma_Y, \nu)$ and any nonnegative measurable function $c: X \times Y \to \Rn$ satisfies some continuity conditions \citep[see e.g.][Theorem 4.1]{villani2009optimal}. Then we say $\sigma^* \in \mC(\mu, \nu)$ is an optimal transport plan if and only if 
\begin{align*}
    \sigma^* \in \argmin_{\sigma \in \mC(\mu, \nu)}\Biggl\{\int_{X \times Y} c(x, y) \, \mathrm{d} \sigma(x, y)\Biggr\}.
\end{align*}
\end{definition}

\subsection{The Generalized Costs}
Now we are ready to introduce our construction. For every practical problem, the action space could be naturally divided into several groups, which then forms a partition of $\mA$, denoted by $\{\mA_\ell\}_{\ell = 1}^L$. Therefore, for each fixed state $s$, when switching from $\pi_\t{o}(\cdot|s)$ to $\pi_\t{n}(\cdot|s)$, the learning cost is to consider the probability mass that is transported \emph{between} different components of $\{\mA_\ell\}_{\ell = 1}^L$. While the transaction cost  focuses on the probability mass that moves \textit{within} each component of $\{\mA_\ell\}_{\ell = 1}^L$. We elaborate the intuition in the coming paragraphs

\textbf{Learning cost.}
For any $s \in \mS$, let $a^s_\ell := \pi_\t{o}(\mA_\ell|s)$, $b^s_\ell := \pi_\t{n}(\mA_\ell|s)$, then we immediately have the following decomposition
\begin{align*}
    \pi_\t{o}(\cdot|s) = \sum_{\ell=1}^L a^s_\ell \pi_{\t{o}, \ell}(\cdot|s) \quad \text{and } \quad \pi_\t{n}(\cdot|s) = \sum_{\ell=1}^L b^{s}_\ell \pi_{\t{n}, \ell}(\cdot|s) 
\end{align*}
where $\pi_{\t{o}, \ell}(\cdot|s) := \pi_\t{o}(\cdot|s) I_{\mA_\ell}/a^{s}_{\ell}$ and $\pi_{\t{n}, \ell}(\cdot|s) := \pi_\t{n}(\cdot|s) I_{\mA_\ell}/b^{s}_{\ell}$ are conditional distributions on $\mA_\ell$. Then if $a^{s}_\ell \not= b^{s}_\ell$, we need to transport $|a^{s}_\ell - b^{s}_\ell|$ amount of mass in or out of $\mA_\ell$, which is captured by the following two measures on $\mA$: 
\begin{align}
    \rho^{s} := \sum_{\ell = 1}^L (a^{s}_\ell - b^{s}_\ell)_+ \pi_{\t{o}, \ell}(\cdot|s) \quad \text{and} \quad \eta^{s} :=  \sum_{\ell = 1}^L (b^{s}_\ell - a^{s}_\ell)_+ \pi_{\t{n}, \ell}(\cdot|s)  \label{def: surplusmass} ,
\end{align}
with $(a^{s}_\ell - b^{s}_\ell)_+ := (a^{s}_\ell - b^{s}_\ell)I_{\{a^{s}_\ell - b^{s}_\ell \geq 0\}}$ and $(b^{s}_\ell - a^{s}_\ell)_+ := (b^{s}_\ell - a^{s}_\ell)I_{\{b^{s}_\ell - a^{s}_\ell\geq 0\}}$. If we further define $\tau^{s}_\ell:= a^{s}_\ell \wedge b^{s}_\ell$. Then $\rho^{s}$ determines how much mass above $\tau^{s}_\ell$ should be moved from $\pi_o(\cdot|s)$ in each $\mA_\ell$. Similar intuition applies to $\eta^{s}$. Thus, any feasible transport plan between $\rho^s$ and $\eta^s$, i.e. 
\begin{align}
    \gamma^{s} \in \mC(\rho^{s}, \eta^{s}), \label{def: LearningPlan}
\end{align} 
would lead to the first step of transportation between $\pi_\t{o}^{s}$ and $\pi_\t{n}^{s}$, i.e. mass moving (while the second step would be shape matching in each component), and the induced cost during this cross-component transportation models the learning cost. Specifically, we define the learning cost as
\begin{align}
    \mL_{c_1}\bigl(\pi_\t{o}(\cdot|s), \pi_\t{n}(\cdot|s)\bigr):= \int_{\mA \times \mA} c^{s}_1(x, y) \, \mathrm{d}\gamma^{s}(x, y) \label{Def: LearningCost}, 
\end{align}
where $c^{s}_1: \mA \times \mA \to \Rn$ is the cost function measures the similarity/distance between two actions. 

\textbf{Transaction cost.}
The above transport plan $\gamma^{s}$ guarantees that $\pi_\t{o}(\cdot|s)$ has the same amount of mass as $\pi_\t{n}(\cdot|s)$ by moving across different components. Then inside each $\mA_\ell$ with $a^{s}_\ell \leq b^{s}_\ell$, we also need to properly rearrange the mass within $\mA_\ell$ , such that the mass has same distribution as $\pi_\t{n}(\cdot|s)$, as the second step of a plan, and the cost incurred by this within-component rearrangement is transaction cost. Such rearrangement can be described by
\begin{align}
    \lambda^{s} := \sum_{\ell=1}^L \tau_{\ell}^s \lambda^{s}_\ell, \label{def: TransactionMeasure}  
\end{align}
where each $\lambda^{s}_\ell \in \mC\bigl(\pi_{\t{o}, \ell}(\cdot|s), \pi_{\t{n}, \ell}(\cdot|s)\bigr)$. With another cost function $c^{s}_2: \mA \times \mA \to \Rn$, the transaction cost is defined as the induced cost during this within-component rearrangement:
\begin{align}\label{Def: TransactionCost}
    \mT_{c_2}\bigl(\pi_\t{o}(\cdot|s), \pi_\t{n}(\cdot|s)\bigr) &:= \int_{\mA \times \mA} c^{s}_2(x, y) \, \mathrm{d}\lambda^{s}(x, y).
\end{align}
Moreover, the following proposition justifies that the combination of the two steps produces a feasible transportation between $\pi_\t{o}^{s}$ and $\pi_\t{n}^{s}$. 
\begin{proposition}[Feasibility of the proposed transport plan]\label{lemma: FeasibleTransportPlan}
    Given old policy $\pi_\t{o}$ and an candidate new policy $\pi_\t{n}$. For each fixed $s\in \mS$, suppose $\gamma^s$ and $\lambda^s$ are defined as~\eqref{def: LearningPlan} and~\eqref{def: TransactionMeasure}, we have $ \gamma^{s} + \lambda^{s} \in \mC(\pi_\t{o}^{s}, \pi_\t{n}^{s})$.
\end{proposition}
Proposition \ref{lemma: FeasibleTransportPlan} assures that any feasible transport plan $\gamma^s$ and $\lambda^s$ will lead to a feasible transport plan between $\pi_\t{o}(\cdot|s)$ and $\pi_\t{b}(\cdot|s)$ and lead to a transport switching cost via~\eqref{Def: LearningCost} and~\eqref{Def: TransactionCost}. In fact, it is possible to be more ambitious by choosing $\gamma^s$ and $\lambda^s$ to be the optimal/near-optimal transport plan. In the following proposition, we demonstrate that the formulation of transport switching cost we defined in~\eqref{def: LearningCost} and~\eqref{def: TransactionCost} can be seen as the optimal value of~\eqref{Def: LearningCost} and~\eqref{Def: TransactionCost} for specific choice of cost functions $c_1^s$ and $c_2^s$.

\begin{proposition}[Respective optimality of the cost terms]\label{prop: Recovery}
    Let $L = 2$, i.e.\ $\mA = \mA_1 \cup \mA_2$ and $\mA_1 \cap \mA_2 = \emptyset$. For each fixed $s \in \mS$, we define $c_1^s(x, y) = \sum_{i, j = 1}^2 I_{\mA_i \times \mA_j}(x, y)$ and $c_2^s(x, y) \equiv 1$, for any $(x, y) \in \mA \times \mA$. In this case, we have 
    \begin{align*}
        \eqref{def: LearningCost} &= \min\Bigl\{\int_{\mA \times \mA} c^{s}_1(x, y) \, \mathrm{d}\gamma^{s}(x, y): \gamma^s \in \mC(\rho^s, \eta^s)\Bigr\}, \\ 
        \eqref{def: TransactionCost} &= \sum_{\ell = 1}^2 \tau_\ell^s   \min \Bigl\{\int_{\mA \times \mA} c_2^s(x, y) \, \mathrm{d} \lambda_{\ell}^s(x, y): \lambda^{s}_\ell \in \mC\bigl(\pi_{\t{o}, \ell}(\cdot|s), \pi_{\t{n}, \ell}(\cdot|s)\bigr)\Bigr\}.
    \end{align*}
\end{proposition}

Before discussing the main theoretical result (i.e. Proposition~\ref{prop:ot}), we need to explicitly emphasize that, although the above generalized costs simultaneously incorporate two directions of extensions compared to \eqref{def: LearningCost} and \eqref{def: TransactionCost}, what Proposition~\ref{prop:ot} is concerned with is only the extension of \eqref{def: LearningCost} and \eqref{def: TransactionCost} by allowing $L>2$, which is formally defined by
\begin{align*}
F(\pi_{\t{o}}(\cdot|s), \pi_{\t{n}}(\cdot|s)) 
&:= c_l\mL(\pi_{\t{o}}(\cdot|s), \pi_{\t{n}}(\cdot|s)) +c_t\mT(\pi_{\t{o}}(\cdot|s), \pi_{\t{n}}(\cdot|s)), \quad \textrm{where}\\
\mL(\pi_{\t{o}}(\cdot|s), \pi_{\t{n}}(\cdot|s)) 
&:= \sum_{\ell = 1}^L (a_\ell^s - b_\ell^s)_+, \quad \textrm{and}\\
\mT(\pi_{\t{o}}(\cdot|s), \pi_{\t{n}}(\cdot|s)) 
&:= \sum_{\ell = 1}^L a_\ell^s \wedge b_\ell^s.
\end{align*}
In such a way, for any $L\geq 2$, $\mL(\pi_{\t{o}}(\cdot|s), \pi_{\t{n}}(\cdot|s)) $ captures the total amount of mass to move across different elements of the partition $\{\mA_\ell\}_{\ell=1}^L$, and $\mT(\pi_{\t{o}}(\cdot|s), \pi_{\t{n}}(\cdot|s)) $ depicts the total amount of mass remaining in each regions.

And for completeness, in the following we can re-state the Proposition~\ref{prop:ot} with formal OT terminologies:

\begin{proposition}[Optimality of the proposed cost]
Given old policy $\pi_\t{o}$, candidate policy $\pi_\t{n}$ and any $\{\mA_\ell\}_{\ell = 1}^L$ as a partition of $\mA$, i.e. $\mA = \cup_{\ell = 1}^L \mA_\ell$ and $\mA_i \cap \mA_j = \emptyset$ for all $i \neq j$. Let $c(x, y) = c_l \sum_{i \neq j} I_{\mA_i \times \mA_j}(x, y) + c_t \sum_{\ell = 1}^L I_{\mA_\ell \times \mA_\ell}(x, y)$ with constants $c_l \geq c_t$. Then for any $s \in \mS$, 
we have
\begin{align*}
F(\pi_\t{o}(\cdot|s), \pi_\t{n}(\cdot|s))
= \min \Bigl\{\int_{\mA \times \mA} c(x, y) \, \mathrm{d}\sigma^s(x, y): \sigma^s \in \mathcal{C}(\pi_\t{o}(\cdot|s), \pi_\t{n}(\cdot|s))\Bigr\}.
\end{align*}
In particular, when $L = 2$, the left-hand side of the above recovers exactly the cost in \eqref{def: AggregateCost:F}, with learning cost and transaction costs chosen as \eqref{def: LearningCost} and \eqref{def: TransactionCost}, respectively. 
\end{proposition}

\subsection{The Graphical Interpretation of the Cost Family} Finally we provide a visual explanation of the cost family defined in \eqref{def: AggregateCost} by the Figure~\ref{fig: Costfamily}.

\begin{figure}
\centering
\begin{tikzpicture}[scale = 0.7,
    node distance=5mm,
    input/.style={fill=cyan!20, minimum size=0.1cm, font = \small},
    operator/.style={fill=pink!60, minimum size=0.1cm},
  ]
    \node[input, inner sep=0.05cm] (i1) at (0,0.3) {$\mL_{s_1}$};
    \node[input, inner sep=0.05cm] (i2) at (1,0.3) {$\mT_{s_1}$};
    \node[] at (2,0.5) {$\ldots$};
    \node[input, inner sep=0.05cm] (i5) at (3,0.3) {$\mL_{s_{\Lambda}}$};
    \node[input, inner sep=0.05cm] (i6) at (4,0.3) {$\mT_{s_{\Lambda}}$};
    
    \node[operator] (h1) at (0.5,1.5) {$\mathrm{Id}$};
    \node[operator] (h3) at (3.5,1.5) {$\mathrm{Id}$};
    \node[operator] (o1) at (2,2.5) {$\sigma$};
    
    
    \node[font = \small] (b1) at (2,3.6) {$C(\pi_\mathrm{o}, \pi_{\mathrm{n}})$};
    
    \foreach \i in {1}
    \foreach \h in {1}
    \draw[teal, ->, line width=0.5pt] (i\i) -- node [text width=0.5cm, midway,above, font = \tiny, xshift = -0.1cm, yshift = -0.2cm]{$c_l$} (h\h);
    
    \foreach \i in {2}
    \foreach \h in {1}
    \draw[teal, ->, line width=0.5pt] (i\i) -- node [text width=0.5cm, midway,above, font = \tiny, xshift = 0.3cm, yshift = -0.2cm]{$c_t$} (h\h);
    
    \foreach \i in {5}
    \foreach \h in {3}
    \draw[teal, ->, line width=0.5pt] (i\i) -- node [text width=0.5cm, midway,above, font = \tiny, xshift = -0.2cm, yshift = -0.2cm]{$c_l$} (h\h);
    
    \foreach \i in {6}
    \foreach \h in {3}
    \draw[teal, ->, line width=0.5pt] (i\i) -- node [text width=0.5cm, midway,above, font = \tiny, xshift = 0.3cm, yshift = -0.2cm]{$c_t$} (h\h);
    
    \foreach \h in {3}
    \foreach \o in {1}
    \draw[teal, ->, line width=0.5pt] (h\h) -- node [text width=0.5cm, midway,above, font = \tiny, xshift = 0.3cm, yshift = -0.1cm]{$f({s}_k)$} (o\o);
    
    \foreach \h in {1}
    \foreach \o in {1}
    \draw[teal, ->, line width=0.5pt] (h\h) -- node [text width=0.5cm, midway,above, font = \tiny, xshift = -0.3cm, yshift = -0.1cm]{$f({s}_1)$} (o\o);
    
    \foreach \o in {1}
    \foreach \b in {1}
    \draw[teal, ->, line width=0.5pt] (o\o) -- (b\b);
    
    \draw [decorate,decoration={brace,amplitude=5pt,raise=4pt},yshift=120pt]
    (-0.4,-4) -- (-0.4,-2.5) node [black,midway,xshift = -0.6cm, yshift=0cm] {$F$};
\end{tikzpicture}
\caption{Proposed cost function family. Here we use the shorthands $\mL_{ s_i}:=\mL(\pi_{\t{o}}(\cdot|s_i), \pi_{\t{n}}(\cdot|s_i))$ and $\mT_{s_i}:=\mT(\pi_{\t{o}}(\cdot|s_i), \pi_{\t{n}}(\cdot|s_i))$.}
\label{fig: Costfamily}
\end{figure}

\section{Further Details on Numerical Experiments}
\label{sec: details_alg}
In this section we provide additional details in the algorithm of NAC as well as implementation techniques in related experiments.

\subsection{Offline Evaluation}
The offline evaluation method used in Algorithm \ref{alg: ac} is presented here as Algorithm \ref{alg: eva}, which is mainly inspired by the widely used offline fitted-Q evaluation \citep{riedmiller2005neural,munos2008finite,le2019batch,ma2023sequential}.

\begin{algorithm}[tb]
   \caption{Offline Net Value Evaluation}\label{alg: eva}
{\bfseries Input:} Offline data $\mathcal{D}$,  initial values in target net Q-function parameters $\{\phi_j^\prime\}_{j\in[M]}$, net Q-function parameters $\{\phi_i\}_{i\in[M]}$, policy $\pi$, learning rates $\rho_{nq}, \rho_\text{stb}$.
\begin{algorithmic}[1]
   \REPEAT
   \STATE Sample a mini-batch $B=\{(s, a, r, s^\prime)\}$ from $\mathcal{D}$
   \STATE Generate $a^\prime\sim\pi(\cdot|s^\prime)$, compute update target by\\
   $y(r, s^\prime) = r - (1-\gamma)C(\pi_0,\pi)\nonumber + \gamma  \min\limits_{j\in[M]}Q_{N,\phi_j^\prime}(s^\prime,a^\prime)$
   \STATE For each $i\in[M]$, update net Q-function $\{Q_{N,\phi_i}\}_{i\in[M]}$ by\\
   $\phi_i \leftarrow  \phi_i - \rho_{nq} \nabla_{\phi_{i}} \sum\limits_{(s, a, r, s^\prime)\in B} [\{y(r, s^\prime) - Q_{N,\phi_i}(s, a)\}^2]$
   \STATE Update target net Q-function by
   $\phi_j^\prime \leftarrow  \rho_\text{stb} \phi_j^\prime + (1-\rho_\text{stb}) \phi_j$
   \UNTIL{Convergence criterion met}
   \STATE $Q_N^{\pi}(s_0,a):= \min\limits_{i\in[M]} Q_{N,\phi_i}(s_0, a)$, for all $a\in\mathcal{A}$
\end{algorithmic}
{\bfseries Output:}  $V_N^{\pi}({s}_0) = \mathbb{E}_{{a}\sim\pi(\cdot|\mathbf{s}_0)}[Q_N^{\pi}({s}_0, {a})]$
\end{algorithm}

\subsection{Stopping of the Algorithms}
For evaluation purpose in Algorithm \ref{alg: eva}, since it is used either for a fine evaluation of the given old policy or the finally found new policy, and such numerical results are directly used for comparisons between such two policies, we need both evaluation process to nearly converge, which only needs the total number of epochs (each epoch contains 1000 evaluation/training steps) to be large.

On the other hand, the case of the policy learning process in Algorithm \ref{alg: ac}, i.e. line 2-8, is more complicated. First, due to offline settings, especially with quite weak old policy as the teacher, the sample distribution of transition tuples in a given offline dataset can be very different from one generated by an optimal policy. If the total number of epochs are too high, not only the later training epochs are possibly not contributing to improving the policy, but also the loss in either net values or net Q-networks may diverge due to over-fitting. Motivated by such observations, we introduce a set of stopping criterion, which contains the following several requirements:

First, we set a threshold named ``epochs\_stop", which is the least number of epochs for policy learning, and we never stop the training before the epoch number reaches ``epochs\_stop". Second, as presented in Algorithm \ref{alg: stop}, we terminate when the current new policy either significantly improves over the old policy or has been even worse for consecutive 2 epochs. This ensures that the NAC training part will be appropriately stopped even when the old policy is optimal or highly suboptimal. All hyper-parameters will be explicitly provided in Section \ref{sec: details_alg}.

\begin{algorithm}[tb]
   \caption{Stopping criterion}\label{alg: stop}
{\bfseries Input:} The list of average estimated net values of new policies in  the last 2 epochs $[v_1, v_2]$. Net value of the old policy $v_o$. Net value increase rate upper bound $\alpha>0$, increase upper bound $b_u>0$, decrease lower bound $b_d>0$, stopping flag $\beta=0$.
\begin{algorithmic}[1]
    \IF{$v_0 > 0$}
        \IF{$v_1 > (1+\alpha)v_0$ and $v_2 > (1+\alpha)v_0$}
            \STATE $\beta=1$
        \ENDIF
    \ELSE
        \IF{$v_1 > 0$ and $v_2 > 0$}
            \STATE $\beta=1$
        \ENDIF
    \ENDIF
    \IF{$v_1 \geq v_0+b_u$ and $v_2 > v_0+b_u$}
            \STATE $\beta=1$
    \ENDIF
    \IF{$v_1 \leq v_0-b_d$ and $v_2 > v_0-b_d$}
            \STATE $\beta=1$
    \ENDIF
\end{algorithmic}
{\bfseries Output:}  Stop the training when $\beta=1$.
\end{algorithm}

\subsection{Training Stability}
As a important universal observation in offline RL, the Q-networks during the training process will be over-optimistic on state-action pairs that do not appear in the offline dataset.
As in half of the cases in experiments, we deal with very weak old policies, some of which even have negative net values. The offline data and the policy cloning effect due to the existence of switching costs will further enhance such over-optimistic behaviour. As a result, to prevent highly volatile updates in each training step, we perform gradient clipping on both the gradient w.r.t. policy parameters and Q-network parameters. Those values are also reported in the next subsection.
\subsection{Hyper-Parameters}
In this subsection we provide hyper-parameters for each distinct experiment settings. Note that the implementation is also inspired by the Spinningup project \citep{SpinningUp2018}.

\begin{table}
\caption{Shared hyper-parameters.}
\label{tab: hyper_shared}
\begin{center}
\begin{tabular}{ll}
\toprule
Parameter  & value   \\
\midrule 
        Number of repetitions & 10 \\
        Offline data size &  1000000 \\
        Batch sample size & 256 \\
        Train from scratch & False\\
        $c_l$ & 5\\
        Random seeds & $\{4, 5, ..., 13\}$ \\
        Number of Monte Carlo state samples for cost function in training & 10\\
        Number of Monte Carlo state samples for cost function in evaluation & 10000\\
        Steps per epoch & 1000\\
        Number of epochs in training & 100\\
        Number of epochs in evaluation & 50\\
        Discount rate $\gamma$ & 0.99\\
        Learning rate & 0.0003\\
        $M$ & 2\\
        Number of Monte Carlo action samples for net value estimate in evaluation & 10000\\
        Maximum length of one episode & 1000\\
        Maximum 2-norm for gradient in Q-networks & 1\\
        Epochs\_stop & 20\\
        Net value increase upper bound $b_u$ & 50\\
        Net value decrease bound $b_d$ & 10\\
        Optimizer & Adam \citep{kingma2014adam}\\
\bottomrule
\end{tabular}
\end{center}
\end{table}

\begin{table}
\caption{Hyper-parameters for Ant-v4, (sub)optimal old policy.}
\label{tab: hyper_ant}
\begin{center}
\begin{tabular}{ll}
\toprule
Parameter  & value   \\
\midrule 
        Maximum 2-norm for gradients in net values & 1 \\
        Number of Monte Carlo action samples for net value estimate in training & 1000\\
        Net value increase rate upper bound $\alpha$ & 1\\
        $c_t$ & $\{0, 1\}$\\
\bottomrule
\end{tabular}
\end{center}
\end{table}

\begin{table}
\caption{Hyper-parameters for HalfCheetah-v4, (sub)optimal old policy.}
\label{tab: hyper_hc}
\begin{center}
\begin{tabular}{ll}
\toprule
Parameter  & value   \\
\midrule 
        Maximum 2-norm for gradients in net values & 5 \\
        Number of Monte Carlo action samples for net value estimate in training & 2000\\
        Net value increase rate upper bound $\alpha$ & 0.15\\
        $c_t$ & $\{0, 0.1\}$\\
\bottomrule
\end{tabular}
\end{center}
\end{table}

\begin{table}
\caption{Hyper-parameters for Hopper-v4, (sub)optimal old policy.}
\label{tab: hyper_hopper}
\begin{center}
\begin{tabular}{ll}
\toprule
Parameter  & value   \\
\midrule 
        Maximum 2-norm for gradients in net values & 1 \\
        Number of Monte Carlo action samples for net value estimate in training & 1000\\
        Net value increase rate upper bound $\alpha$ & 1\\
        $c_t$ & $\{0, 0.1\}$\\
\bottomrule
\end{tabular}
\end{center}
\end{table}

\begin{table}
\caption{Hyper-parameters for SUMO-RL that are different from Gymnasium, (sub)optimal old policy.}
\label{tab: hyper_sumo}
\begin{center}
\begin{tabular}{ll}
\toprule
Parameter  & value   \\
\midrule 
        Environment name & sumo-rl-v0\\
        Intersection type & Simple intersection\\
        Offline data size & 100000\\
        Maximum 2-norm for gradients in net values & None \\
        Maximum 2-norm for gradients in net Q-functions & None \\
        Number of epochs in training & 150\\
        Maximum length of one episode & 100\\
        Steps per epoch & 400\\
        Net value increase rate upper bound $\alpha$ & 0.5\\
        Net value increase upper bound $b_u$ & 5.0\\
        Net value decrease bound $b_d$ & 5.0\\
        $(c_l, c_t)$ & $\{(0.5, 0.01), (5.0, 0.1), (10.0, 1.0)\}$\\
\bottomrule
\end{tabular}
\end{center}
\end{table}

\subsection{Compute Resources}
Our experiments ran on a single Precision 7875 Tower workstation, with AMD Ryzen Threadripper PRO 7945WX CPU (64 MB cache, 12 cores, 24 threads, 4.7GHz to 5.3GHz), NVIDIA RTX 6000 Ada GPU. In the training process, the memory needed was around 7.2GB. The time to get all results was within one week.

\subsection{The SUMO-RL Environment}
As we implement our algorithm in a novel real problem, the traffic control problem, we first introduce the general problem setting, and then describe the concerned technical details of the environment we used in the experiments.

\subsubsection{The Traffic Control Problem}
Efficient traffic signal control is crucial for urban areas, as it directly impacts congestion levels, travel times, fuel consumption, and overall quality of life for residents. Frequent policy switches due to online algorithms can potentially influence traffic safety \cite{han2023leveraging}. In contrast, our approach focuses on training a traffic signal control policy using only offline data, striking a balance between maximizing future expected rewards (speeds of vehicles) and minimizing potential switching cost (such as temporary traffic congestion, update of facilities). To this end, we employ the SUMO-RL environment \cite{sumorl} designed for developing and assessing traffic control algorithms in realistic urban scenarios. 

\subsubsection{Environment Details}
Simulation of Urban MObility (SUMO) is a widely used simulation system of transport and large road networks \cite{lopez2018microscopic}. And it has been adopted to an RL-friendly interface by SUMO-RL \cite{sumorl} for both single agent and multi-agent scenarios. While we refer readers to the documentations of both open-source projects for complete set of details, here we introduce core technical settings of SUMO-RL interface.

\textbf{State.} Given a complex traffic network, there are several/single intersection(s). The agent(s) is allowed to observe the traffic information and adaptively change the traffic lights phases. Especially, the observation, i.e. the state variables, is a long vector with both discrete and continuous entries. The state variables include the phase of the current traffic lights, a binary variable `min\_green' indicating whether minimum green light time has passed, the density of vehicles in each lane of each directions (North, South, East, West, each containing multiple lanes) as well as the respective number of queuing vehicles in each lane.

\textbf{Action.} Actions correspond to distinct traffic light configurations in each intersection, i.e. the ways one can change the lights. And to model the real life cases, the is another constant `yellow\_time' so that after a new action being executed, the light will be turned yellow for such amount of seconds before turning green/red.

\textbf{Reward.} For different goals of a problem setting, one can define different instant rewards, for example, the total delay around intersections (the summation of all queuing times of all nearby vehicles). While in our experiment, we focus on the average speed of all nearby vehicles.

In our implementation, by default we only provide textual output. However, it is ready to generate graphical user interface to visualize results. Further details can be found in the aforementioned documentations and our source code.

\subsection{Limitations}
Throughout the paper we considered a general cost formulation relying on optimal transport (OT). We paid specific attention to costs within this class, specified by \eqref{def: LearningCost} and \eqref{def: TransactionCost}. One can always consider more general costs, but this instantiation of the  costs is theoretically justified, probably the simplest to explain, and already provides a more fine-grained quantification for the switching cost  compared to existing approach (local and global switching cost), conveying the key ideas.

\section{Further Experimental Results}
\label{sec: furtherexp}
In this section we provide additional results in Gymnasium environments, when the coefficient of transaction cost $c_t$ takes nonzero values, followed by ablation study.

\subsection{Additional Results} We implement NAC when $c_t\in\{0.1, 1\}$ to provide comparisons to the above results when $c_t=0$, the results of which are presented in Table \ref{tab: furtherresults}.

\begin{table}
\caption{Additional performance of the NAC algorithm on various Gymnasium benchmarks. 1st column: environment considered. 2nd column: dim($\mS$). 3rd column: dim($\mA$). 4th column: (sub)optimality of the old policy $\pi_\t{o}$. 5-7th columns: performance measures. The performance measures are meant as mean $\pm$ std. The average net values of old policies are $-14.2$ (Ant-v4), $-52.8$ (HalfCheetah-v4), $17.0$ (Hopper-v4).}
\label{tab: furtherresults}
\begin{center}
\begin{tabular}
{r@{\hspace{0.2cm}}r@{\hspace{0.1cm}}r@{\hspace{0.2cm}}r@{\hspace{0.2cm}}r@{\hspace{0.2cm}}c@{\hspace{0.2cm}}c@{\hspace{0.2cm}}c}
\toprule
Environment $\mM$ & $d_\mS$ & $d_\mA$ & Old policy $\pi_\t{o}$ & $c_t$ & Improvement& Switch proportion  & Responsible rate \\
\midrule 
        Ant-v4 & $27$ & $8$ &suboptimal & 1 & 55.2 $\pm$ 23.2 & 90.0\%  & 100.0\% \\
      HalfCheetah-v4 & 17 & 6 & suboptimal & 0.1 & 24.1 $\pm$ 8.9 & 100.0\%  & 100.0\% \\
            Hopper-v4 & 11 & 3 & suboptimal & 0.1 & 52.0 $\pm$ 19.8 & 100.0\%   & 100.0\% \\
\midrule              
        Ant-v4                  & $27$ & $8$ & optimal & 1 & / & 0.0\%  & 100.0\% \\
       HalfCheetah-v4               & 17 &  6  & optimal & 0.1 & / & 0.0\%  & 100.0\%  \\
             Hopper-v4             & 11 & 3 & optimal & 0.1 & / & 0.0\%   & 100.0\%  \\
\bottomrule
\end{tabular}
\end{center}
\end{table}

\subsection{Ablation Study.} Here we mainly want to understand if the scale of cost functions influence the policy learning performance in different environments. As seen in Table \ref{tab: results} and \ref{tab: furtherresults}, we increased $c_t$ from $0$ to $0.1$ or $1$. Especially note that, to guarantee fair comparisons, in each environment, apart from $c_t$, all hyperparameters are kept exactly the same, independent of $c_t$ values or the (sub)optimality of the given old policies, which can be checked according to Table \ref{tab: hyper_ant}, \ref{tab: hyper_hc} and \ref{tab: hyper_hopper}. Finally, by comparing the results environment-wise, we can see that, when increasing $c_t$ by an appropriate value, the performance in both optimal and suboptimal old policy cases are similar or slightly better than when $c_t=0$. Such observation is important, as it shows that NAC training process is robust to different scaling of cost functions, which makes it applicable to different scenarios.


\section{Proofs}\label{sec: proofs}
This section is about our proofs.
\subsection{Proof of Lemma \ref{lemma: existence}}\label{proof: existence}
\begin{proof}
    By Assumption \ref{asmp: compact}, both $\{V^\pi({{s}}_0)\}_{\pi\in\Pi}$ and $\{C(\pi_\text{o}, \pi)\}_{\pi\in\Pi}$ are compact in the topology generated by open sets in $\mathbb{R}$, which then implies that both sets are sequentially compact. Then consider the set of resulting net values $\{V_N^\pi({{s}}_0)\}_{\pi\in\Pi}$. For any sequence $(z_i)_{i \geq 1}\subseteq\{V_N^\pi({{s}}_0)\}_{\pi\in\Pi}$, by definition of net values, we must have sequences $(x_i)_{i\geq1}\subseteq\{V^\pi({{s}}_0)\}_{\pi\in\Pi}$ and $(y_i)_{i\geq 1}\subseteq\{C(\pi_\text{o}, \pi)\}_{\pi\in\Pi}$, such that $z_i = x_i - y_i$ for all $i$ (and specifically, for each given $i$, $x_i, y_i$ corresponds to value and cost of the same policy). Since $\{V^\pi({{s}}_0)\}_{\pi\in\Pi}$ is sequentially compact, there exists a subsequence $(x_{i_j})_{j \geq 1}$ of $(x_i)_{i \geq 1}$ such that for some $x\in \{V^\pi({{s}}_0)\}_{\pi\in\Pi}$, $x_{i_j} \to x$ as $j \to +\infty$. On top of such sequence of indices $(i_j)_{j \geq 1}$, Since $\{C(\pi_\text{o}, \pi)\}_{\pi\in\Pi}$ is sequentially compact, there exists a further subsequence of $(i_j)_{j \geq 1}$, denoted as $(i_{j_k})_{k \geq 1}$ such that $y_{i_{j_k}} \to y$, as $k\to +\infty$, for some $y\in\{C(\pi_\text{o}, \pi)\}_{\pi\in\Pi}$. So we immediately know $(z_{i_{j_k}})_{k \geq 1}$, as a subsequence of $(z_i)_{i \geq 1}$, converges to $x-y$. Then $\{V_N^\pi({{s}}_0)\}_{\pi\in\Pi}$ is sequentially compact, and especially attains its supremum by some policy in $\Pi$.
\end{proof}
\subsection{Proof of Lemma \ref{lemma: nontrivial}}\label{proof: nontrial}
\begin{proof}
For an arbitrarily given MDP, all we need is to construct a counter-example, so that the optimal policy is not switch-optimal. So we just  discuss how to design such a counter-example. Given current behaviour policy $\pi_\text{o}$ and some fixed initial state $\mathbf{s}_0$, let's consider an arbitrary policy $\pi$ and the optimal policy in value function $\pi^\ast$. By definition of optimality, we know $V^{\pi^\ast}({s}_0)\geq V^{\pi}({s}_0)$, and especially we denote the gap by $M$, i.e. $M:=V^{\pi^\ast}({s}_0)-V^{\pi}({s}_0)$. Now as long as in some problem settings, the cost function $C$ is larger in $\pi^\ast$, i.e. $C(\pi_\text{o}, \pi^\ast) > C(\pi_\text{o}, \pi)$, it could then be the case that $V_N^\pi$ dominates that of $\pi^\ast$. To be more specific, whenever $C(\pi_\text{o}, \pi^\ast) - C(\pi_\text{o}, \pi) >M$, we would have $V_N^\pi > V_N^{\pi^\ast}$, making ${\pi^\ast}$ not switch-optimal.
\end{proof}

\subsection{Proof of Proposition \ref{lemma: po_state}}\label{proof: po_state}
\begin{proof}
Given an MDP $\mathcal{M}=(\mathcal{S}, \mathcal{A}, P, R, \gamma)$, with a fixed initial state $\mathbf{s}_0$, let's consider the following example, which is also depicted in Figure \ref{fig: illustrationofproof}. To begin with, let $\mathcal{S}=\{{s}_\alpha, {s}_\beta\}$, and $\mathcal{A}=\{{a}_\text{self}, {s}_\text{alt}\}$. At any state, taking action ${a}_\text{self}$ means trying to stay in the same state, guaranteed by letting $P({s}|{s}, {a}_\text{self})=1$ for any ${s}\in\mathcal{S}$. Meanwhile, we also let $P({s}_\alpha|{s}_\beta, {a}_\text{alt})=P({s}_\beta|{s}_\alpha, {a}_\text{alt})=1$, so that whenever the action ${a}_\text{alt}$ is taken at any state, the environment would transit the agent to the other state. Then we define the rewards as $r({s}_\alpha, {a}_\text{self}) = r({s}_\beta, {a}_\text{alt}) = 1$, while $r({s}_\alpha, {a}_\text{alt}) = r({s}_\beta, {a}_\text{self}) = 0$. Together with the definition of the transition dynamics, it just means that the reward is $1$ if and only if the state to arrive at is ${s}_\alpha$, and vanishes otherwise. Finally, choose $\gamma=0.99$ for simplicity.

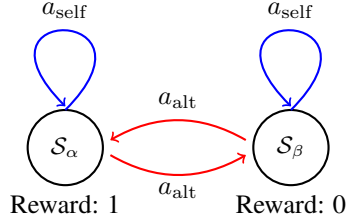
\begin{figure}
\centering
\begin{tikzpicture}[scale=1]
\draw[thick] (0,0) circle (0.5cm);
\draw[thick] (3,0) circle (0.5cm);
\draw[thick,->, red] (0.6,-0.1) to [out = -30, in = -150] (2.4,-0.1);
\draw[thick,->, red] (2.4,0.1) to [out = 150, in = 30] (0.6,0.1);
\node[font = \small] at (0,0) {$\mathcal{S}_\alpha$};
\node[font = \small] at (3,0) {$\mathcal{S}_\beta$};

\draw[thick,->, blue] (0,0.5) to [out=45,in=135,loop,looseness=100] (0,0.55);
\draw[thick,->, blue] (3,0.5) to [out=45,in=135,loop,looseness=100] (3,0.55);

\node[above] at (0,1.6) {$a_{\mathrm{self}}$};
\node[above] at (3,1.6) {$a_{\mathrm{self}}$};
\node[above] at (1.5,0.4) {$a_{\mathrm{alt}}$};
\node[above] at (1.5,-0.8) {$a_{\mathrm{alt}}$};

\node[above] at (0,-1) {Reward: 1};
\node[above] at (3,-1) {Reward: 0};
\end{tikzpicture}
\caption{Illustration of the constructed MDP.}
\label{fig: illustrationofproof}
\end{figure}

With such environment, any policy $\pi$ from the pool of feasible policies $\Pi$ takes the form of $\pi=\{\pi(\cdot|{s}_\alpha), \pi(\cdot|{s}_\beta)\}$. Now let's construct a specific example, where the current behaviour policy is $\pi_\text{o}$, with $\pi_\text{o}({a}|{s})=1/2$ for any $({s}, {a})\in \mathcal{S}\times \mathcal{A}$. We focus on the case when there is no transaction cost. In addition, for any candidate policy $\pi$ to switch to from $\pi_\text{o}$, the learning cost is high whenever in one state $\pi$ is a stochastic policy. And actually let such high learning cost to be $500$. On the other hand, when $\pi$ is always deterministic at any state, the learning cost is low. Especially, if in all states, the action to execute is the same, cost is $25$, while $50$ if actions to take are different in different states.

By the above settings, all the candidate policies (excluding $\pi_\text{o}$) can be divided into 2 groups: $\{\pi_\text{n1}, \pi_\text{n2}, \pi_\text{n3}, \pi_\text{n4}\}$, and $ \Pi\setminus \{\pi_\text{o}, \pi_\text{n1}, \pi_\text{n2}, \pi_\text{n3}, \pi_\text{n4}\}$, where 
$\pi_\text{n1}({a}_\text{self}|{s})=1$ for any ${s}\in\mathcal{S}$; $\pi_\text{n2}({a}_\text{alt}|{s})=1$ for any ${s}\in\mathcal{S}$; $\pi_\text{n3}({a}_\text{self}|{s}_\alpha)=1$ and $\pi_\text{n3}({a}_\text{alt}|{s}_\beta)=1$; $\pi_\text{n4}({a}_\text{alt}|{s}_\alpha)=1$ and $\pi_\text{n4}({a}_\text{self}|{s}_\beta)=1$.

Now let's calculate the values and net values for all policies in the first group. First we consider $\pi_\text{n1}$. If initial state is ${s}_\alpha$, since  one would always take ${a}_\text{self}$, the agent would stay in ${s}_\alpha$, to earn reward $1$ recursively for all rounds, leading to 
\begin{align*}
    V^{\pi_\text{n1}}({s}_\alpha)
    =\sum_{i=0}^\infty 0.99^i \cdot 1
    = \frac{1}{1-0.99}
    =100.
\end{align*}

On the other hand, if ${s}_0 = {s}_\beta$, then the agent would stay in ${s}_\beta$, receiving zero rewards, so that $V^{\pi_\text{n1}}({s}_\beta)=0$. By such way, we can continue to know $V^{\pi_\text{n2}}({s}_\alpha)=49.75, V^{\pi_\text{n2}}({s}_\beta)=50.25$ for $\pi_\text{n2}$; $V^{\pi_\text{n3}}({s}_\alpha)=V^{\pi_\text{n3}}({s}_\beta)=100$ for $\pi_\text{n3}$; and $V^{\pi_\text{n4}}({s}_\alpha)=V^{\pi_\text{n4}}({s}_\beta)=0$ for $\pi_\text{n4}$. Recall there is low costs of switch $25$ for first 2 policies and $50$ for the next 2, we finally know $V_N^{\pi_\text{n1}}({s}_\alpha)=75$, $V_N^{\pi_\text{n1}}({s}_\beta)=-25$ for $\pi_\text{n1}$; $V_N^{\pi_\text{n2}}({s}_\alpha)=24.75, V_N^{\pi_\text{n2}}({s}_\beta)=25.25$ for $\pi_\text{n2}$; $V_N^{\pi_\text{n3}}({s}_\alpha)=V_N^{\pi_\text{n3}}({s}_\beta)=50$ for $\pi_\text{n3}$; and $V_N^{\pi_\text{n4}}({s}_\alpha)=V_N^{\pi_\text{n4}}({s}_\beta)=-50$ for $\pi_\text{n4}$.

The let's consider the second group. Due to the problem settings with horizon $\gamma=0.99$ and the maximal immediate reward of $1$, the highest possible return from any initial state is bounded by $100$, then for all policies in the second group, the net value in any state is bounded by $-400$ due to the high cost, and could never be switch-optimal in any state, given the results in the first group.

For complete comparisons, we could know that for the current policy $\pi_\text{o}$, net values share the same numbers as its values, which are $V^{\pi_\text{o}}({s}_\alpha)=V^{\pi_\text{o}}({s}_\beta)=50$.

By comparing the net values among all policies, we observe that, if ${s}={s}_\alpha$, then the switch-optimal is $\pi_\text{n1}$; while, if ${s}={s}_\beta$, then either the switch-optimal is $\pi_\text{n3}$ or that we just don't make a switch.

\end{proof}

\subsection{Proof of Proposition \ref{lemma: po_action}}\label{apdx: po_action} 
\begin{proof}
The proof would be quite straight-forward if we follow the example settings in the Proof in Appendix \ref{proof: po_state} in the above. As discussed there, we know that $\pi_\text{n1}$ is switch-optimal at initial state ${s}={s}_\alpha$. Then to compute the corresponding net Q-functions, first consider the case $Q_N^{\pi_\text{n1}}({s}_\alpha, {a}_\text{self})$. If ${a}_\text{self}$ is executed at state ${s}_\alpha$, according to the transition dynamics, the agent would remain in such state until termination, leading to $Q^{\pi_\text{n1}}({s}_\alpha, {a}_\text{self})=100$, and then $Q_N^{\pi_\text{n1}}({s}_\alpha, {a}_\text{self})=75$. On the other hand, if ${a}_\text{alt}$ is executed at state ${s}_\alpha$, it would arrive at ${s}_\beta$ and remain there, having $Q^{\pi_\text{n1}}({s}_\alpha, {a}_\text{alt})=0$, and $Q_N^{\pi_\text{n1}}({s}_\alpha, {a}_\text{alt})=-25$. Following the same idea, we can immediately know $Q^{\pi_\text{n3}}({s}_\alpha, {a}_\text{alt})=99$ and $Q_N^{\pi_\text{n3}}({s}_\alpha, {a}_\text{alt})=49>Q_N^{\pi_\text{n1}}({s}_\alpha, {a}_\text{alt})=-25$, showing that $\pi_{\mathrm{n1}}$ is not switch-optimal for every action ${a}\in\mathcal{A}$.
\end{proof}

\subsection{Proof of Proposition \ref{prop: evaluation}}
\begin{proof}
    The proof consists of two parts. First we want to show that such net Bellman operator $B^\pi$ is a contraction map on $\mG(\mS \times \mA, \|\cdot\|_{\infty})$ under $\|\cdot\|_{\infty}$-norm. Then we show such repeated implementations of the contraction lead to the unique evaluation. Without loss of generality, we focus on the proof when state and action spaces are finite, while we note that the proof can be readily extended to the continuous case. 
    
    Now, take arbitrarily two net Q-functions $Q_N^{\pi_1}, Q_N^{\pi_2}\in \mG(\mS \times \mA, \|\cdot\|_{\infty})$, we observe that
    \begin{align*}
        &~\|B^\pi Q_N^{\pi_1} - B^\pi Q_N^{\pi_2}\|_\infty\\
        =&~ \max_{(\b s, \b a) \in \mS \times \mA}|(R(s,  a) - (1-\gamma)C(\pi_0,\pi) + \gamma \E_{{s}' \sim P(\cdot| s,a)} \E_{{a}'\sim\pi(\cdot| {s}')}Q_N^{\pi_1}({s}', {a}'))\\
        &\quad - (R( s, a) - (1-\gamma)C(\pi_0,\pi) + \gamma \E_{{s}' \sim P(\cdot|s, a)} \E_{ a'\sim\pi(\cdot|  s')}Q_N^{\pi_2}( s', a'))|\\
        =&~ \gamma\max_{( s, a) \in \mS \times \mA}| [\E_{ s' \sim P(\cdot| s,a)} \E_{ a'\sim\pi(\cdot|  s')}Q_N^{\pi_1}( s',  a')]
        - [\E_{ s' \sim P(\cdot| s,  a)} \E_{{a}'\sim\pi(\cdot| s')}Q_N^{\pi_2}( s', {a}')]|\\
        =&~ \gamma\max_{(s, a) \in \mS \times \mA}| \E_{{s}' \sim P(\cdot|s,\b a)} \E_{ a'\sim\pi(\cdot|  s')}[Q_N^{\pi_1}( s',  a') - Q_N^{\pi_2}({s}', a')]|\\
        \leq&~ \gamma\max_{({s}, a)\in\mS\times\mA}|Q_N^{\pi_1}( s, a) - Q_N^{\pi_2}(s, a)| = \gamma\|Q_N^{\pi_1} - Q_N^{\pi_2}\|_\infty.
    \end{align*}
    By such we know $B^\pi$ is a contraction mapping, whenever $\gamma\in[0,1)$. After that we review the following theorem.
    \begin{theorem}[Banach Fixed-point Theorem]
        For a non-empty complete metric space $(X,d)$ with contraction $\mathcal{T}:X\to X$, $\mathcal{T}$ has a unique fixed point $\mathbf{x}^\ast\in X$. In addition, starting from arbitrary point $\mathbf{x}_0\in X$, and define a new sequence as $\{\mathbf{x}_n\} = \{\mathcal{T}\mathbf{x}_{n-1}\}$, then we have $\lim_{n\to \infty}\mathbf{x}_n = \mathbf{x}^\ast$.
    \end{theorem}
    Since the concerned net $Q$-functions are defined based a finite horizon MDP with bounded reward function, they are contained in $\mG(\mS \times \mA, \|\cdot\|_{\infty})$. Given that $\mG(\mS \times \mA, \|\cdot\|_{\infty})$ is complete when metrized by $\|\cdot\|_{\infty}$-norm \citep[see e.g.][pp.121]{folland1999real}, starting with a $Q_N^0 \in \mG(\mS \times \mA, \|\cdot\|_{\infty})$, $Q_N^k$ converges to a fixed point $Q_N^\ast \in \mG(\mS \times \mA, \|\cdot\|_{\infty})$, which, by the definition of fixed point, satisfies the net Bellman equation:
    \begin{align*}
        Q_N^\ast({s}, {a}) = R({s}, {a}) - (1-\gamma)C(\pi_0,\pi) + \gamma \mathbb{E}_{{s}^\prime \sim P(\cdot|{s},{a})} \mathbb{E}_{{a}^\prime\sim\pi(\cdot|, {s}^\prime)}Q_N^\ast({s}^\prime, {a}^\prime)
    \end{align*}
    Still due to the Banach fixed-point theorem, we know the corresponding fixed point $Q_N^\ast$ is unique, which means that $Q_N^\ast = Q_N^\pi$. Then such iterations of backup would converge, i.e. $\lim_{k\to \infty}Q_N^k = Q_N^\pi$.
\end{proof}

\subsection{Proof of Proposition \ref{lemma: FeasibleTransportPlan}}

\begin{proof}
    Writing $\sigma^s:= \gamma^s + \lambda^s$, we only need to show that for any measurable set $G \subseteq \mathcal{A}$, we have $\sigma^s(G \times \mA) = \pi_\t{o}(G|s)$ and $\sigma^s(\mA \times G) = \pi_\t{n}(G|s)$. Recall that $\gamma^s \in \mC(\rho^s, \eta^s)$ and $\lambda^s = \sum_{\ell = 1}^L \tau_\ell^s \lambda_\ell^s$ with each $\lambda_\ell^s \in \mC\bigl(\pi_\t{o}(\cdot|s), \pi_\t{n}(\cdot|s)\bigr)$, thus we have
    \begin{align*}
        \sigma^s(G \times \mathcal{A}) &= \gamma^s(G \times \mathcal{A}) + \sum_{\ell = 1}^L(a_\ell \wedge b_\ell)\lambda_\ell(G \times \mathcal{A})
        = \rho^s(G) + \sum_{\ell = 1}^L (a_\ell \wedge b_\ell) \pi_{o, \ell}(G|s).  
    \end{align*}
    Then we take into the explicit form of $\rho^s$ defined in \eqref{def: surplusmass}, it yields that
    \begin{align*}
        \sigma^s(G \times \mathcal{A}) &= \sum_{\ell = 1}^L (a_\ell - b_\ell)_+ \pi_{o, \ell}(G|s) + \sum_{\ell = 1}^L (a_\ell \wedge b_\ell) \pi_{o, \ell}(G|s) \\  &=  \sum_{\ell = 1}^L a_\ell \pi_{o, \ell}(G|s) = \pi_o(G|s).
    \end{align*}
    A similar calculation can be carried out to obtain $\sigma^s(\mathcal{A} \times G) = \pi_n(G|s)$. Hence the claim is verified. 
\end{proof}

\subsection{Proof of Proposition \ref{prop: Recovery}}
\begin{proof}
When $L = 2$, recall the definition of $\rho^s$ and $\eta^s$ in \eqref{def: surplusmass}, we have
\begin{align*}
\rho^s &= (a_1^s - b_1^s)_+ \pi_{\t{o}, 1}(\cdot|s) + (a_2^s - b_2^s)_+ \pi_{\t{o}, 2}(\cdot|s), \\
\eta^s &= (b_1^s - a_1^s)_+ \pi_{\t{n}, 1}(\cdot|s) + (b_2^s - a_2^s)_+ \pi_{\t{n}, 2}(\cdot|s).
\end{align*}
Since $\pi_\t{n}(\cdot|s)$ and $\pi_\t{n}(\cdot|s)$ are probability measures, exactly one of two cases must hold: either (1) $a_1^s \geq b_1^s$ and $b_2^s \geq a_2^s$, or (2) $a_1^s \leq b_1^s$ and $b_2^s \leq a_2^s$. Consequently, $\rho^s$ and $\eta^s$ take one of two forms: either $\rho^s = (a_1^s - b_1^s) \pi_{\t{o}, 1}(\cdot|s)$ and $\eta^s = (b_2^s - a_2^s) \pi_{\t{n}, 2}(\cdot|s)$, or $\rho^s = (a_2^s - b_2^s) \pi_{\t{o}, 2}(\cdot|s)$ and $\eta^s = (b_1^s - a_1^s) \pi_{\t{n}, 1}(\cdot|s)$. Without loss of generality, we assume the former case. Given that $\gamma^s$ is a feasible transport plan between $\rho^s$ and $\eta^s$, it must concentrate on $\mA_1 \times \mA_2 \subseteq \mA_1 \times \mA$, thus we have 
\begin{align*}
     \int_{\mA \times \mA} c^{s}_1(x, y) \, \mathrm{d}\gamma^{s}(x, y) =\gamma^s(\mA_1 \times \mA)  \stackrel{(a)}{=} (a_1^s - b_1^s)\pi_{\t{o}, 1}(\mA_1) = a_1^s - b_1^s = \eqref{def: LearningCost}.          
\end{align*}
where equality (a) is by the definition of feasible transport plan in Definition \ref{def:feasible}. Note this holds for all feasible transport plan, it naturally becomes the optimal transport cost. 
    
    As for the transaction cost, since that $\lambda_{\ell}^s$ represents a feasible transport plan between the distributions $\pi_{\t{o}, \ell}(\cdot|s)$ and $\pi_{\t{n}, \ell}(\cdot|s)$ and both distributions are concentrated on $\mA_\ell$, we have $\lambda_\ell^s$ must concentrate on $\mA_\ell \times \mA_\ell$. Consequently, by the form of $\lambda^s$ defined \eqref{def: TransactionMeasure}, it follows that
    \begin{align*}
        \int_{\mA \times \mA} c^{s}_2(x, y) \, \mathrm{d}\lambda^{s}(x, y) =& \tau_1^s\lambda_1^s(\mA_1 \times \mA_1) + \tau_2^s\lambda_2^s(\mA_2 \times \mA_2)
        = \tau_1^s + \tau_2^s = \eqref{def: TransactionCost}.
    \end{align*}
    Again, since the calculation above does not depend on the specific choice of each $\lambda_\ell^s$, it coincides with the case when we choosing each $\lambda_\ell^s$ as the optimal one.  
\end{proof}

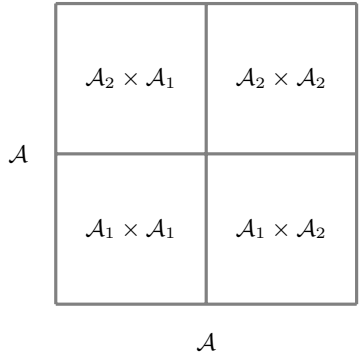
\begin{figure}
    \centering
\begin{tikzpicture}
    \draw[step=2cm,color=gray, very thick] (0,0) grid (4,4);
    \node[font = \small] at (1,3) {$\mA_2 \times \mA_1$};
    \node[font = \small] at (3,3) {$\mA_2 \times \mA_2$};
    \node[font = \small] at (1,1) {$\mA_1 \times \mA_1$};
    \node[font = \small] at (3,1) {$\mA_1 \times \mA_2$};

    \node[font = \small] at (-0.5,2) {$\mA$};
    \node[font = \small] at (2,-0.5) {$\mA$};

\end{tikzpicture}
    \caption{The product action space $\mA \times \mA$ is partitioned into four components.}
    \label{fig: support}
\end{figure}

\subsection{Proof of Proposition \ref{prop:ot}}
\begin{proof}
We show the statement by going through several steps. First we provide a novel perspective to decompose the cost of any feasible transport plan; then we calculate the cost of one specific transport plan, which corresponds to our proposal in the transport cost by \eqref{def: LearningCost} and \eqref{def: TransactionCost} up to any $L$ (in the sense that the plan exactly attains a cost of $F(\pi_\t{o}(\cdot|s), \pi_\t{n}(\cdot|s))$); finally we verifies that any other feasible transport plan can never achieve a cost strictly lower than the above case.

\tb{Step 1 (Decomposition of the cost).} For any feasible transport plan $\sigma^s=\gamma^s+\lambda^s$, by definition the cost is
\begin{align*}
&\int_{\mA \times \mA} c(x, y) \, \mathrm{d}\sigma^s(x, y)\\
=~& \int_{\mA \times \mA} \sum_{l=1}^L [I_{\mA_l\times\mA}(x,y)~ c(x, y)] \, \mathrm{d}\sigma^s(x, y)\\
=~& \sum_{l=1}^L \int_{\mA \times \mA}  I_{\mA_l\times\mA}(x,y)~ c(x, y) \, \mathrm{d}\sigma^s(x, y)\\
=~& \sum_{l=1}^L \int_{\mA_l \times \mA}  c(x, y) \, \mathrm{d}\sigma^s(x, y),
\end{align*}
where the first inequality is due to the fact that, for any valid partition $\{\mA_l\}_{l=1}^L$ of the set $\mA$, and any $x,y\in\mA$, one always has 
\begin{align*}
\sum_{l=1}^L I_{\mA_l\times\mA}(x,y)\equiv 1.    
\end{align*}

Now for convenience, we introduce a new notation by
\begin{align*}
d_l^s := \int_{\mA_l \times \mA}  c(x, y) \, \mathrm{d}\sigma^s(x, y),
\end{align*}
which immediately leads to the following simplified decomposition:
\begin{align*}
\int_{\mA \times \mA} c(x, y) \, \mathrm{d}\sigma^s(x, y)
= \sum_{l=1}^L d_l^s.
\end{align*}
To begin with, for each $l\in[L]$, $d_l^s$ just represents the total cost due to any transport which has a starting location inside $\mA_l$, which provides a novel and straightforward view of the total cost due to any feasible transport plan $\sigma$. In addition, as $\{\mA_l\}_{l=1}^L$ is a partition, the decomposition of total cost is naturally constructed, allowing a clear comparison between any two different transport plans, which will be showcased in the following steps.

\tb{Step 2 (The case of the proposed transport cost).} 
Now we specify one transport plan $\sigma^s=\gamma^s+\lambda^s$, where $\gamma^s$ is any feasible transport plan defined by \eqref{def: LearningPlan}, i.e. $\gamma$ is one feasible way to transport all the surplus mass $\{(a_\ell^s-b_\ell^s)_+\}_{\ell=1}^L$ respectively out of $\{\mA_\ell\}_{\ell=1}^L$; while $\lambda$ is one feasible choice according to \eqref{def: TransactionMeasure}, i.e. a way for rearrangement. As Proposition \ref{lemma: FeasibleTransportPlan} has already established that $\eta^s$ is a feasible transport plan, its cost is well-defined, which is further computed by the sum of all $d_\ell^s$, where, for each $\ell\in[L]$,
\begin{align*}
d_\ell^s
&=\int_{\mA_\ell \times \mA} c(x,y) \, \mathrm{d}\sigma^s(x, y)\\
&=\int_{\mA_\ell \times \mA} c_l \left[\sum_{i \neq j} I_{\mA_i \times \mA_j}(x, y)\right] + c_t \left[\sum_{i = 1}^L I_{\mA_i \times \mA_i}(x, y)\right] \, \mathrm{d}\sigma^s(x, y)\\
&=c_l \int_{\mA_\ell \times \mA} \left[\sum_{i \neq j} I_{\mA_i \times \mA_j}(x, y)\right] \, \mathrm{d}\sigma^s(x, y) + c_t \int_{\mA_\ell \times \mA} \left[\sum_{i = 1}^L I_{\mA_i \times \mA_i}(x, y)\right] \, \mathrm{d}\sigma^s(x, y)\\
&= c_l \left[\sum_{j \neq \ell} \int_{\mA_\ell \times \mA} I_{\mA_\ell \times \mA_j}(x, y) \, \mathrm{d}\sigma^s(x, y) \right] + c_t \int_{\mA_\ell \times \mA} I_{\mA_\ell \times \mA_\ell}(x, y) \, \mathrm{d}\sigma^s(x, y) \\ 
&= c_l \left[ \sum_{j=1}^L \int_{\mA_\ell \times \mA} I_{\mA_\ell \times \mA_j}(x, y) \, \mathrm{d}\sigma^s(x, y) \right] - c_l \int_{\mA_\ell \times \mA} I_{\mA_\ell \times \mA_\ell}(x, y) \, \mathrm{d}\sigma^s(x, y) + c_t \int_{\mA_\ell \times \mA} I_{\mA_\ell \times \mA_\ell}(x, y) \, \mathrm{d}\sigma^s(x, y)\\ 
&= c_l \int_{\mA_\ell \times \mA} \sum_{j=1}^L I_{\mA_\ell \times \mA_j}(x, y) \, \mathrm{d}\sigma^s(x, y) + (c_t - c_\ell) \int_{\mA_\ell \times \mA} I_{\mA_\ell \times \mA_\ell}(x, y) \, \mathrm{d}\sigma^s(x, y)\\ 
&= c_l \int_{\mA_\ell \times \mA} I_{\mA_\ell \times \mA}(x, y) \, \mathrm{d}\sigma^s(x, y) + (c_t - c_\ell) \int_{\mA_\ell \times \mA} I_{\mA_\ell \times \mA_\ell}(x, y) \, \mathrm{d}\sigma^s(x, y).
\end{align*}
Now for the first term, by the property of any feasible transport plan described in Definition~\ref{def:feasible}, we know
\begin{align*}
c_l \int_{\mA_\ell \times \mA} I_{\mA_\ell \times \mA}(x, y) \, \mathrm{d}\sigma^s(x, y)
&= c_l \sigma^s(\mA_\ell \times \mA)\\
&= c_l (\gamma^s+\lambda^s)(\mA_\ell \times \mA)\\
&= c_l [\gamma^s (\mA_\ell \times \mA)+\lambda^s (\mA_\ell \times \mA)]\\
&= c_l [\rho^s (\mA_\ell)+\lambda^s (\mA_\ell \times \mA)]\\
&= c_l \left[\sum_{i=1}^L (a_i^s-b_i^s)_+\pi_{\textrm{o}, i}(\mA_\ell|s) \right] +  c_l \left[\sum_{i=1}^L \tau_i^s \lambda_i^s \right] (\mA_\ell\times\mA)\\
&= c_l \left[\sum_{i=1}^L (a_i^s-b_i^s)_+\pi_{\textrm{o}, i}(\mA_\ell|s) \right] + c_l  \left[\sum_{i=1}^L (a^{s}_i \wedge b^{s}_i) \pi_{\t{o},i}(\mA_\ell|s) \right]\\
&= c_l [(a_\ell^s-b_\ell^s)_+ + a^{s}_\ell \wedge b^{s}_\ell].
\end{align*}
And then for the second term,
\begin{align*}
(c_t - c_\ell) \int_{\mA_\ell \times \mA} I_{\mA_\ell \times \mA_\ell}(x, y) \, \mathrm{d}\sigma^s(x, y)
&= (c_t - c_\ell) \sigma^s(\mA_\ell \times \mA_\ell)\\
&= (c_t - c_\ell) [\gamma^s(\mA_\ell \times \mA_\ell) + \lambda^s(\mA_\ell \times \mA_\ell) ].
\end{align*}
Due to the definition of $\gamma^s$ in \eqref{def: LearningPlan} and the property of $(a_\ell^s - b_\ell^s)_+$, we know there is at most one of $(a_\ell^s - b_\ell^s)$ and $(b_\ell^s - a_\ell^s)$ that is non-zero, and more importantly, for an arbitrarily given $\ell\in[L]$, without loss of generality, if we assume $a_\ell^s - b_\ell^s\geq0$, then we immediately know $(a_\ell^s - b_\ell^s)_+\geq0$, while $(b_\ell^s - a_\ell^s)_+=0$. In addition,
\begin{align*}
\gamma^s (\mA_i, \mA_i) 
&\leq \gamma^s(\mA, \mA_\ell)\\
&= \eta^s(\mA_\ell)\\
&= \sum_{i=1}^L (b_i^s-a_i^s)_+ \pi_{\textrm{n}, i}(\mA_\ell|s)\\
&= (b_\ell^s-a_\ell^s)_+ \pi_{\textrm{n}, \ell}(\mA_\ell|s)
=0\\
&\Rightarrow
\gamma^s(\mA_i, \mA_i)=0, \quad \textrm{for any}~ i\in[L].
\end{align*}
On the other hand, by similarly noticing that, for any $\ell\in[L]$, $\lambda^{s}_\ell$ vanishes on any $\mA_{\ell^\prime}\times\mA_{\ell^\prime}$ with $\ell^\prime\neq\ell$, we further know
\begin{align*}
\lambda^{s}(\mA_\ell \times \mA_\ell) 
= \sum_{i=1}^L \tau_{i}^s \lambda^{s}_i(\mA_\ell \times \mA_\ell)
= \tau_{\ell}^s \lambda^{s}_\ell(\mA_\ell \times \mA_\ell)
= \tau_{\ell}^s
= a_\ell^s \wedge b_\ell^s\\
\Rightarrow
(c_t - c_\ell) \int_{\mA_\ell \times \mA} I_{\mA_\ell \times \mA_\ell}(x, y) \, \mathrm{d}\sigma^s(x, y)
=(c_t - c_\ell) a_\ell^s \wedge b_\ell^s.
\end{align*}
Combining the above two terms,
\begin{align*}
d_\ell^s = c_l [(a_\ell^s-b_\ell^s)_+ + a^{s}_\ell \wedge b^{s}_\ell] + (c_t - c_\ell) a_\ell^s \wedge b_\ell^s
=  c_l (a_\ell^s-b_\ell^s)_+ + c_t (a_\ell^s \wedge b_\ell^s).
\end{align*}
Therefore, 
\begin{align*}
\int_{\mA \times \mA} c(x, y) \, \mathrm{d}\sigma^s(x, y) 
=\sum_{l=1}^L d_l^s
= c_l \sum_{\ell = 1}^L (a_\ell^s - b_\ell^s)_+ +  c_t \sum_{\ell = 1}^L a_\ell^s \wedge b_\ell^s = F(\pi_\t{o}(\cdot|s), \pi_\t{n}(\cdot|s)).
\end{align*}

\tb{Step 3 (Costs of other transport plans).} As we have already derived in the above step, for any feasible transport plan $\sigma$ (no longer to be restricted to the proposed ones due to \eqref{def: LearningPlan} and \eqref{def: TransactionMeasure}), we always have
\begin{align*}
\int_{\mA \times \mA} c(x, y) \, \mathrm{d}\sigma^s(x, y)
&= \sum_{l=1}^L d_l^s \quad \textrm{with}\\
d_\ell^s &= c_l \left[\sum_{j \neq \ell} \int_{\mA_\ell \times \mA} I_{\mA_\ell \times \mA_j}(x, y) \, \mathrm{d}\sigma^s(x, y) \right] + c_t \int_{\mA_\ell \times \mA} I_{\mA_\ell \times \mA_\ell}(x, y) \, \mathrm{d}\sigma^s(x, y).
\end{align*}
As in each problem setting, $\pi_\t{o}, \pi_\t{n}$ and $\{\mA_\ell\}$ are all given, we know $\{\pi_\t{o}(\mA_\ell)\}_{\ell=1}^L$ are all nonnegative constants with a sum of $1$. Further notice that
\begin{align*}
\sum_{j \neq \ell} \int_{\mA_\ell \times \mA} I_{\mA_\ell \times \mA_j}(x, y) \, \mathrm{d}\sigma^s(x, y) + \int_{\mA_\ell \times \mA} I_{\mA_\ell \times \mA_\ell}(x, y) \, \mathrm{d}\sigma^s(x, y)
&= \int_{\mA_\ell \times \mA} \sum_{j=1}^L  I_{\mA_\ell \times \mA_j}(x, y) \, \mathrm{d}\sigma^s(x, y)\\
&= \int_{\mA_\ell \times \mA} I_{\mA_\ell \times \mA}(x, y) \, \mathrm{d}\sigma^s(x, y)\\
&= \sigma(\mA_\ell \times \mA)\\
&= \pi_\t{o}(\mA_\ell|s).
\end{align*}
So if we denote
\begin{align*}
J_\ell^s:=\sum_{j \neq \ell} \int_{\mA_\ell \times \mA} I_{\mA_\ell \times \mA_j}(x, y) \, \mathrm{d}\sigma^s(x, y),
\end{align*}
we immediately have
\begin{align*}
d_\ell^s &= c_l J_\ell^s + c_t [\pi_\t{o}(\mA_\ell|s) - J_\ell^s].
\end{align*}
Before proceeding, we first show an intermediate statement:
\begin{align*}
J_\ell^s \geq (a_\ell^s - b_\ell^s)_+.
\end{align*}
This is because: first, if $a_\ell^s - b_\ell^s\leq0$, then as $J_\ell^s$ is defined through a sum of nonnegative measures, thus the statement holds. Then if $a_\ell^s - b_\ell^s>0$,
\begin{align*}
J_\ell^s
&=\sum_{j \neq \ell} \int_{\mA_\ell \times \mA} I_{\mA_\ell \times \mA_j}(x, y) \, \mathrm{d}\sigma^s(x, y)\\
&= \pi_\t{o}(\mA_\ell|s) - \int_{\mA_\ell \times \mA} I_{\mA_\ell \times \mA_\ell}(x, y) \, \mathrm{d}\sigma^s(x, y)\\
&= a_\ell^s - \sigma^s(\mA_\ell \times \mA_\ell).
\end{align*}
In addition, since 
\begin{align*}
\sigma^s(\mA_\ell \times \mA_\ell)\leq \sigma^s(\mA \times \mA_\ell)  =\pi_\t{n}(\mA_\ell|s) = b_\ell^s,
\end{align*}
we finally know $J_\ell^s \geq (a_\ell^s - b_\ell^s) = (a_\ell^s - b_\ell^s)_+$.

Now by knowing $d_\ell^s = c_l J_\ell^s + c_t [\pi_\t{o}(\mA_\ell|s) - J_\ell^s]$, together with the fact that $J_\ell^s \geq (a_\ell^s - b_\ell^s)_+$ and $c_l\geq c_t$, we know $d_\ell^s$ attains its minimum when $_\ell^s = (a_\ell^s - b_\ell^s)_+$, and such minimum is
\begin{align*}
c_l (a_\ell^s - b_\ell^s)_+ + c_t [\pi_\t{o}(\mA_\ell|s) - (a_\ell^s - b_\ell^s)_+]
&= c_l (a_\ell^s - b_\ell^s)_+ + c_t [a_\ell^s - (a_\ell^s - b_\ell^s)_+]\\
&= c_l (a_\ell^s - b_\ell^s)_+ + c_t (a_\ell^s\wedge b_\ell^s).
\end{align*}
As a result, for any feasible transport plan $\sigma$, its corresponding cost is lower bounded by
\begin{align*}
c_l \sum_{\ell = 1}^L (a_\ell^s - b_\ell^s)_+ +  c_t \sum_{\ell = 1}^L a_\ell^s \wedge b_\ell^s = F(\pi_\t{o}(\cdot|s), \pi_\t{n}(\cdot|s)),
\end{align*}
which shows that $F(\pi_\t{o}(\cdot|s), \pi_\t{n}(\cdot|s))$ is the solution to the OT problem.
\end{proof}

\subsection{Global and Local Switching Costs are Special Cases}
\begin{lemma}\label{prop: SCRecover}
Let $\pi_1$ and $\pi_2$  be two policies on a state space $\mS$ with finite cardinality. When considering policy switching from $\pi_1$ to $\pi_2$, recall that the induced global and the local switching costs are defined as 
\begin{align*}
    C^{\mathrm{gl}}{(\pi_1, \pi_2)} = I_{\{\pi_1 \not= \pi_2\}}, \quad C^{\mathrm{loc}}{(\pi_1, \pi_2)} = \sum_{s\in \mS} I_{\{\pi_1(\cdot|s) \not= \pi_2(\cdot|s)\}}.
\end{align*}
Then one can get back $C^{\mathrm{gl}}$ and $C^{\mathrm{loc}}$ as a specific case of the cost family \eqref{def: AggregateCost} by the parameters given in Table~\ref{tab: CoveredSC}.
\end{lemma}
\begin{proof}
    One can recover the global switching cost as follows.
        \begin{align*}
        C(\pi_{\mathrm{1}}, \pi_{\mathrm{2}}) &= \sigma\Bigl(\int_{\mS} f({s}) F\bigl(\pi_1(\cdot|s) ,\pi_2(\cdot|s)\bigr) \, \d\mu(s)\Bigr)
        =I_{\{ \sum_{s\in \mS}\frac{1}{|\mS|} I_{\{\pi_1(\cdot|s) \not= \pi_2(\cdot|s)\}} >0 \}} \\ 
        &=I_{\{\sum_{s\in \mS} I_{\{\pi_1(\cdot|s) \not= \pi_2(\cdot|s)\}}>0\}}
        = C^{\mathrm{gl}}(\pi_1, \pi_2).
    \end{align*}
    One can get back the local switching cost as follows.
    \begin{align*}
        C(\pi_{\mathrm{1}}, \pi_{\mathrm{2}}) &= \sigma\Bigl(\int_{\mS} f({s}) F(\pi_1(\cdot|s) ,\pi_2(\cdot|s)) \, \mathrm{d}\mu({s})\Bigr)
        =|\mS| \sum_{s\in \mS}\frac{1}{|\mS|} I_{\{\pi_1(\cdot|s) \not= \pi_2(\cdot|s)\}} \\
        &=\sum_{s\in \mS} I_{\{\pi_1(\cdot|s) \not= \pi_2(\cdot|s)\}} = C^{\mathrm{loc}}(\pi_1, \pi_2). 
    \end{align*}

\end{proof}

\end{document}